\documentclass[11pt]{article}

\usepackage[preprint]{acl}

\usepackage{times}
\usepackage{latexsym}
\usepackage{tablefootnote}

\usepackage[T1]{fontenc}

\usepackage[utf8]{inputenc}

\usepackage{microtype}

\usepackage{inconsolata}

\usepackage{graphicx}

\usepackage{booktabs}
\usepackage{pifont}

\newcommand{\cmark}{\ding{51}}
\newcommand{\xmark}{\ding{55}}

\usepackage{amsmath}
\usepackage{amssymb}
\usepackage{mathtools}
\usepackage{amsthm}
\usepackage{booktabs}
\usepackage{multirow}
\usepackage[table]{xcolor} 
\usepackage{colortbl}       

\usepackage{microtype}
\usepackage{subcaption}
\usepackage{booktabs} 

\theoremstyle{plain}
\newtheorem{theorem}{Theorem}[section]

\theoremstyle{definition}

\theoremstyle{remark}

\usepackage{algorithm}
\usepackage{algorithmic}
\usepackage{tabularx}

\title{Hybrid Open-Ended Tri-Evolution Makes Better Deep Researcher}

\author{
\bfseries
Hongming Piao$^{1}$\thanks{Equal contribution} \quad
Chi Liu$^{1}$\footnotemark[1] \quad
Mengzhuo Chen$^{1}$ \quad
Yan Shu$^{1}$ \quad Xidong Wang$^{1}$ \\
\bfseries Derek Li$^{1}$ \quad Ying Wei$^{2}$ \quad
Bryan Dai$^{1}$\thanks{Corresponding author}
 \\
$^{1}$IQuest Research \quad $^{2}$Zhejiang University \\
\texttt{\{cxiao, cliu04, cbdai\}@iquestlab.com} \\
\href{https://github.com/IQuestLab/ote}{\texttt{https://github.com/IQuestLab/ote}}
}

\begin{document}
\maketitle
\begin{abstract}
  Deep research and agent evolution serve as de-facto tasks for AI agents in real-world applications toward artificial general intelligence. The former enables autonomous retrieval and integration of information in open-ended environments to tackle open-ended research tasks, yet it is constrained by the static parametric deep research capabilities of agent systems. The latter allows agents to autonomously interact with the environment to gain experiences that evolve model capabilities. However, its effectiveness has been widely validated only on verifiable tasks with standard answers, leaving a gap with open-ended research tasks. To bridge these two critical tasks, we propose the Hybrid Open-Ended Tri-Evolution (HOTE) framework, which leverages hybrid-mode reinforcement learning to facilitate the collaborative evolution of a proposer, solver and judge based on web-scale knowledge, moving toward autonomous evolving agents in open-ended tasks and environments. Extensive experiments on three long-form deep research benchmarks demonstrate that the 8B model trained via HOTE surpasses the strongest static open 8-32B models as well as those trained by state-of-the-art deep research training methods with less time overhead, and further verify that the evolution of all three modules in HOTE is indispensable.
\end{abstract}

\section{Introduction}
\label{sec:introduction}




Deep research, which emphasizes autonomous handling of open-ended, long-cycle, and highly complex information retrieval and integration, has become a de-facto task for AI agents in real-world applications and a step toward artificial general intelligence~\cite{hu2025step,openai2025deepresearch}. Closed-source proprietary systems such as OpenAI Deep Research~\cite{openai2025deepresearch}, Claude Research~\cite{antrophic2025deepresearch}, Kimi-Researcher~\cite{kimi2025deepresearch}, and Grok DeepSearch~\cite{xai2025deepresearch} have demonstrated near-human research capabilities. Meanwhile, the open-source community has also made significant progress in building more comprehensive research workflows and end-to-end training of deep researchers capable of autonomously planning workflows~\cite{schmidgall2025agent,li2025websailor,jin2025search,team2025tongyi,shao2025dr,fang2025cognitive,zheng2025deepresearcher,wu2025webdancer,yao2026researcher,song2025r1}.

Although deep researchers can tackle highly complex research questions by autonomously seeking web-scale knowledge, their parameterized research capabilities are upper bounded by fixed training sets and training strategies. Autonomous interaction with the environment and evolution through experience are regarded as a path toward artificial general intelligence~\cite{liu2025spice}, with self-play offering a promising paradigm for agent evolution, where an agent system learns from feedback acquired through competition with itself. For example, \citet{huang2025r,zhao2025absolute,wang2025cure,chen2025self} proposed agent systems that act as both query proposer and solver, achieving significant results beyond the original training set or even in zero-data scenarios in domains such as mathematics, coding, or general reasoning. To address the limitation that such evolution is constrained by the agent system’s own knowledge, SPICE~\cite{liu2025spice} and Dr. Zero~\cite{yue2026dr} equipped the proposer with a pretrained-scale corpus~\cite{mahabadi2025nemotron,yuan2025naturalreasoning} and a search engine, respectively, taking a step forward toward evolution in open-ended environments. However, they remain limited to tasks that can be verified with deterministic answers. Given that deep researchers in real-world applications often face long-form report generation tasks without clear standard answers~\cite{shao2025dr}, constructing an agent evolution framework for open-ended tasks and environments is crucial.


\begin{figure*}[t]
\centering
\includegraphics[width=0.98\textwidth]{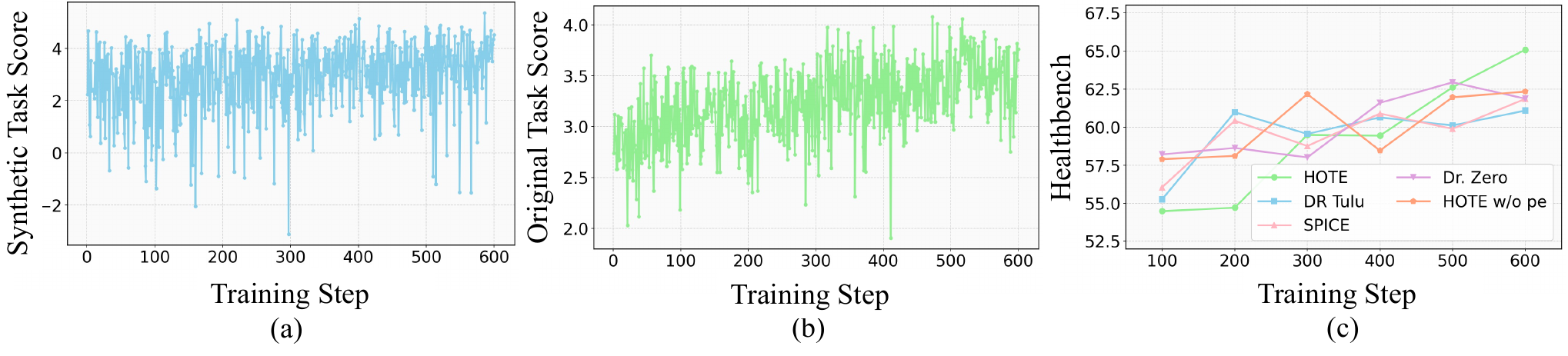}
\caption{During the training of HOTE, (a) the scores for synthetic research tasks remain at the same level; (b) the scores for research tasks from the original training set continuously increase; (c) the scores on Healthbench surpass the baselines and maintain an upward trend.}
\label{fig:intro}
\end{figure*}

To fill the aforementioned gaps, we propose the Hybrid Open-ended Tri-Evolution (HOTE) framework, which consists of three co-evolving modules: proposer, solver, and judge. The solver is responsible for receiving a query, generating a research plan, conducting multi-turn information seeking, integrating information and producing a referenced research report. The judge is responsible for dynamically generating rubrics~\cite{gunjal2025rubrics,viswanathan2025checklists,shao2025dr} that capture the strengths and weaknesses of the solver by comparing multiple solver responses sampled for the same query, and providing rewards for the responses based on these rubrics, thereby removing the dependency on verifiable answers. The proposer is responsible for performing information seeking based on the model weaknesses identified by the judge and proposing challenging yet learnable queries. HOTE uses GRPO~\cite{shao2024deepseekmath} to encourage a game between the solver and proposer, continuously improving response quality and query difficulty. Simultaneously, it employs the judge to dynamically evolve evaluation rubrics, preventing reward hacking, maintaining the learnability of difficult queries, and enabling the proposer to uncover the solver’s weaknesses. Additionally, we propose a dual-mode hybrid training strategy that includes both tool-use and no-tool modes, which achieves mutual benefit between the two modes and significantly improves training efficiency. HOTE effectively maintains the difficulty of synthetic queries during training (Figure~\ref{fig:intro}(a-b)) and outperforms approaches using only the original training set within the same number of training steps (Figure~\ref{fig:intro}(c)). As shown in Figure~\ref{fig:reli}(a), HOTE also facilitates the collaborative progress of both the no-tool and tool-use modes.

In conclusion, our contributions are as follows: 
\begin{itemize}
\item We propose Hybrid Open-ended Tri-Evolution (HOTE), the first deep researcher evolution framework designed for open-ended environments and open-ended tasks, bridging two paths toward artificial general intelligence: deep research and agent evolution. 

\item We design a co-evolution strategy for proposer, solver and judge based on reinforcement learning with hybrid modes. The strategy maintains the challenge and learnability of research tasks for the solver while avoiding reward hacking and achieving the mutual benefit between tool-use mode and no-tool mode.

\item Experimental results on three long-form deep research benchmarks demonstrate that an 8B model trained with HOTE outperforms the strongest open 8-32B models and state-of-the-art deep research training methods with less time overhead, with the co-evolution of all three modules being indispensable.
\end{itemize}

\section{Method}
\label{sec:method}

\subsection{Problem Formulation}

\begin{figure}[t]
\centering
\includegraphics[width=0.49 \textwidth,height=0.07\textheight]{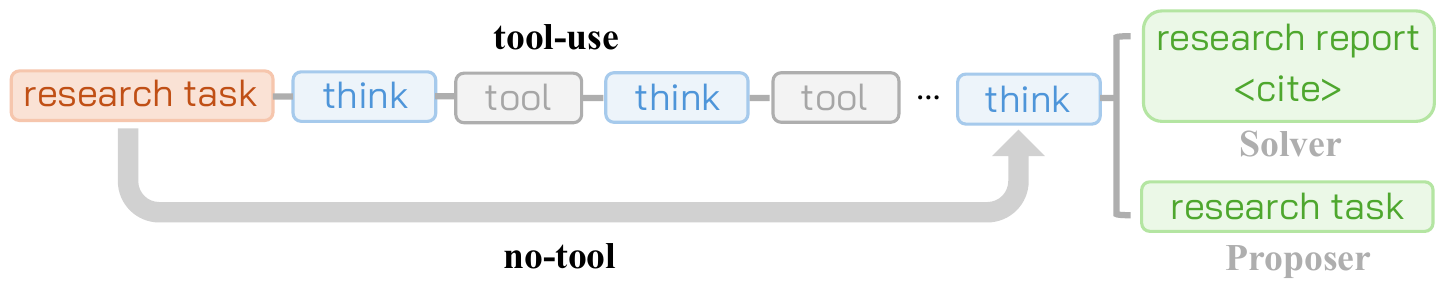}
\caption{The inference paradigm of the solver and the proposer under tool-use and no-tool modes in HOTE.}
\label{fig:problem}
\end{figure}


\begin{figure*}[t]
\centering
\includegraphics[width=0.90 \textwidth]{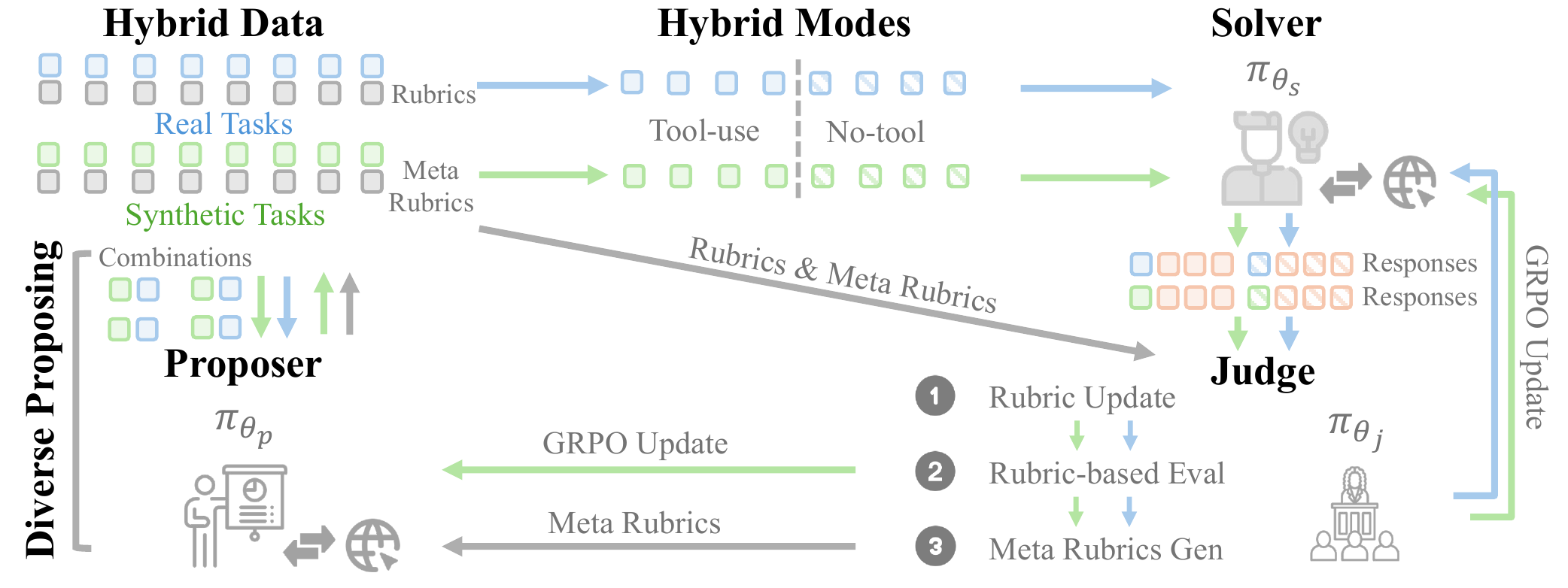}
\caption{The overall training framework of HOTE. At each training step, we utilize hybrid data consisting of both \textcolor[RGB]{78, 149, 217}{real tasks} and \textcolor[RGB]{78, 167, 46}{synthetic tasks} with their corresponding persistent rubrics. Half of the tasks are configured in \textbf{tool‑use} mode and the other half in \textbf{no‑tool} mode. The \textbf{Solver} generates responses in hybrid mode. Based on each task’s existing rubrics and the generated responses, the \textbf{Judge} updates the rubrics, evaluates the responses and generates meta rubrics. The assessment generated by the \textbf{Judge} is used to update the \textbf{Solver}, while the portion corresponding to synthetic tasks is used to update the \textbf{Proposer}. The \textbf{Proposer} performs diverse proposing according to the meta rubrics and different combinations of tasks from the previous step, thereby generating synthetic tasks which use the meta rubrics as persistent rubrics for the next step.}
\label{fig:method}
\end{figure*}


Following~\cite{li2025websailor}, we build our agent on top of a concise and general ReAct framework, which provides a clear baseline for evaluating the model’s intrinsic capabilities and training strategies. The deep research model is a language model (LM) augmented with search tools. Each tool accepts a query along with its arguments and returns textual resources that can be cited in the model’s final answer. Formally, let $\mathcal{T} = \{T_1, T_2, \ldots\}$ represent the set of available tools. 
Each tool $T_k$ accepts a query $q$ together with an optional argument string $\alpha$, 
and returns an observation $o = T_k(q; \alpha)$. The model follows a policy $\pi_\theta$, parameterized by $\theta$, which generates a 
sequence of text $s$ autoregressively. The sequence is initialized as 
$s_0 = x$, where $x$ contains the system prompt and the task description. The model’s action space is defined as
\[
\{\texttt{think}, \texttt{tool}, \texttt{answer}, \texttt{cite}\},
\]
with each action associated with a corresponding protocol token. \texttt{think} (\texttt{<think>$...$</think>}) leverages the language model’s internal reasoning capability to plan subsequent steps based on the current state and available information. \texttt{tool} (\texttt{<call\_tool name=$...$>$...$</call\_tool>}) triggers the invocation of one of several search-related tools. The specific tool is selected via the \texttt{name} attribute, together with tool-dependent arguments omitted here. The textual output produced by the tool is appended to the context for use in later steps. \texttt{answer} (\texttt{<answer>$...$</answer>}) generates the final response and terminates the interaction. \texttt{cite} (\texttt{<cite id=$...$>$...$</cite>}) is embedded within the final answer to annotate claims with citation tags that reference supporting sources.

At each step $i$, the model samples both an action $a_i$ and its associated content or arguments $\zeta_i$, $(a_i, \zeta_i) \sim \pi_\theta(\cdot \mid s_i)$. If $a_i \in \{\texttt{think}, \texttt{answer}, \texttt{cite}\}$, the generated output $\zeta_i$
is appended to the context, yielding $s_{i+1} = s_i \oplus \langle a_i, \zeta_i \rangle$. If $a_i = \texttt{tool}$, the model executes the corresponding tool call, receives the
observation $o_i = T_k(q_i; \alpha_i)$, where $\zeta_i=(q_i, \alpha_i)$, and updates the state as $s_{i+1} = s_i \oplus \langle a_i, \zeta_i, o_i \rangle$. This iterative procedure continues until $a_\tau = \texttt{answer}$, at which point
$\zeta_\tau$ contains the final answer. As shown in Figure~\ref{fig:problem}, within HOTE, both the proposer and the solver perform inference under the same paradigm described above. \textbf{Formulating challenging research tasks for the solver constitutes a research task for the proposer itself}. The key difference is that the proposer does not include the \texttt{cite} action, since proposing research tasks does not require the presentation of citations. In HOTE, both the proposer and the solver operate under two modes: tool-use and no-tool. In the tool-use mode, the model follows the aforementioned inference paradigm. In the no-tool mode, after receiving the initial state $s_0$, the model performs a single \texttt{think} action and then directly produces an \texttt{answer} action. All inference paradigms described above can be controlled through the system prompt. 


\begin{figure*}[t]
\centering
\includegraphics[width=0.98\textwidth]{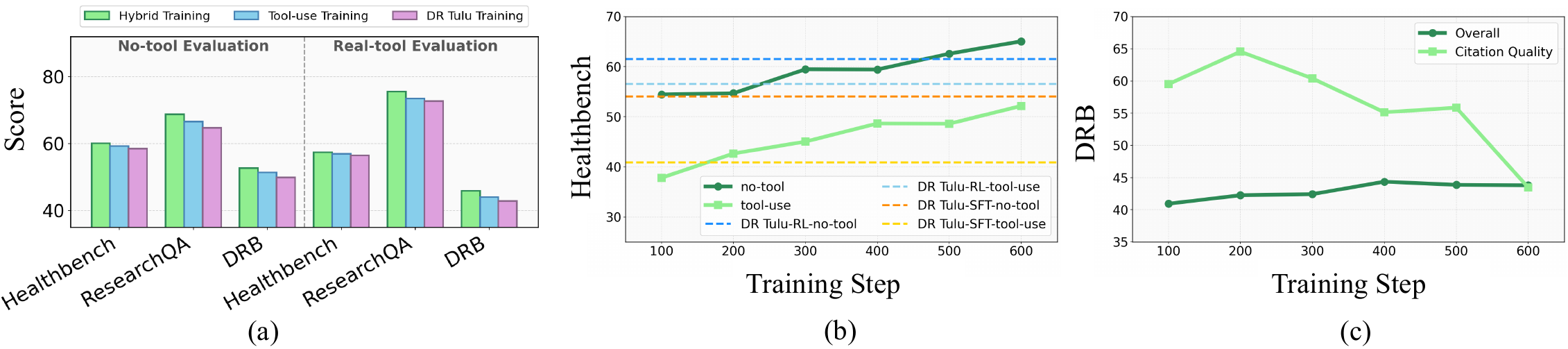}
\caption{(a) The hybrid mode of HOTE outperforms the tool-use mode of HOTE as well as DR Tulu in both no-tool and tool-use modes; (b) Models trained with no-tool mode HOTE and DR Tulu evaluated in no-tool mode on Healthbench achieve higher scores than when evaluated in tool-use mode; (c) When trained with HOTE in no-tool mode, the scores on DRB under tool-use mode decrease after a certain number of steps.}
\label{fig:reli}
\end{figure*}

\subsection{Hybrid Open-ended Tri-evolution}

HOTE is primarily divided into four parts: \textbf{Solver Evolution}, \textbf{Judge Evolution}, \textbf{Proposer Evolution} and \textbf{Dual-mode Hybrid Training Strategy}. Please refer to Figure~\ref{fig:method} and Algorithm~\ref{alg:hote_training} for the overall framework and training pipeline.

\textbf{Solver Evolution}. 
The solver $\pi_{\theta_{s}}$ takes a research task $s_0$ as input and, after performing \texttt{think}-\texttt{tool} interleaved reasoning, generates a long-form research report \texttt{answer} with \texttt{cite} represented by $o$. Thus, the objective of solver evolution is to make the answer better align with the research report requirements $r$, which also serves as the reward in the reinforcement learning and will be further discussed in the judge evolution section. We utilize GRPO~\cite{shao2024deepseekmath} with token-level loss aggregation~\cite{yu2025dapo} to achieve solver evolution, with the goal of:

\begin{equation}
\begin{aligned}
\mathcal{J}_{\text{GRPO}}&(\theta_s)
= \mathbb{E}_{(s_0,\mathcal{R}_{s_0})\sim\mathcal{D},\ \{o_i\}_{i=1}^G \sim \pi_{\theta_s^{\text{old}}}(\cdot \mid s_0)} \\
&\quad \Bigg[\frac{1}{\textstyle \sum_{i=1}^{G}|o_i|}\sum_{i=1}^G
\sum_{t=1}^{|o_i|}
\Big(
\min\big(
r_{i,t}(\theta_s)\,\hat{A}_{i,t}, \\
&\hspace{-2.7em} \mathrm{clip}\!\left(r_{i,t}(\theta_s),\epsilon\right)\,
\hat{A}_{i,t}
\big)
- \beta\,D_{\mathrm{KL}}\!\left(\pi_{\theta_s} \,\|\, \pi_{\theta_s^{\text{ref}}}\right)
\Big)
\Bigg],\\
&\hspace{-2.7em}\text{where}\quad r_{i,t}(\theta)
=
\frac{
\pi_\theta\!\left(o_{i,t}\mid q, o_{i,<t}\right)
}{
\pi_{\theta_{\text{old}}}\!\left(o_{i,t}\mid q, o_{i,<t}\right)
},\\
&\hat{A}_{i,t}
=
\frac{
r_i - \operatorname{mean}\!\left(\{r_i\}_{i=1}^G\right)
}{
\operatorname{std}\!\left(\{r_i\}_{i=1}^G\right)
}.
\end{aligned}
\label{eq:grpo}
\end{equation}
$\{o_i\}_{i=1}^G$ represents a group of responses to the research task $s_0$. $r_i$ denotes the reward obtained by $o_i$. $\mathcal{R}_{s_0}$ represents the rubric set corresponding to $s_0$, which will be explained in detail in the judge evaluation section. The solver continuously progresses toward better long-form research reports based on the reward. We omit the descriptions of other symbols that can be found in~\citet{yu2025dapo}.

\textbf{Judge Evolution}. 
The judge $\pi_{\theta_{j}}$ receives a group of responses $\{o_i\}_{i=1}^G$ for $s_0$ from the solver as input and assigns reward $r_i$ to each response $o_i$ according to the rubric set $\mathcal{R}_{s_0}$ as follows:

\begin{align}
r_i=\frac{\sum_{(R,w)\in \mathcal{R}_{s_0}}w\cdot \text{Judge}_{\pi_{\theta_{j}}}(o_i,R)}{\sum_{(R,w)\in \mathcal{R}_{s_0}}|w|},  
\label{eq:rubric}
\end{align}

where $R$ represents a rubric in $\mathcal{R}_{s_0}$ and $w$ represents its corresponding weight. The judge's reward for each rubric has only $0$ or $\pm1$. Therefore, the evolutionary objective of the judge is to provide more well-founded and discriminative rewards for the responses, ensuring the learning of the solver. As can be seen from Equation~\ref{eq:rubric}, the judge requires extensive inference at each step, so for training efficiency considerations, the judge in HOTE uses a fixed instruction model. In this case, the key to judge evolution shifts to how to drive the evolution of the rubric set $\mathcal{R}_{s_0}$ for $s_0$. Inspired by~\cite{shao2025dr}, given $\mathcal{R}_{s_0}=\mathcal{R}^{\text{persi.}}_{s_0}\cup\mathcal{R}^{\text{active}}_{s_0}$ where $\mathcal{R}^{\text{persi.}}$ contains persistent rubrics of $s_0$ and $\mathcal{R}^{\text{active}}_{s_0}$ contains active rubrics of $s_0$ that can be deleted or added, HOTE prompts the judge to update $\mathcal{R}^{\text{active}}_{s_0}$ based on $\{o_i\}_{i=1}^G$ at each step before assigning rewards as follows:

\begin{align}
\mathcal{R}^{\text{active}}_{s_0}=\text{Update}_{\pi_{\theta_{j}}}(s_0,\{o_i\}_{i=1}^G,\mathcal{R}^{\text{active}}_{s_0}).
\label{eq:judge_evolve}
\end{align}

The judge will generate two types of rubrics: \textbf{positive rubrics} that capture strengths or new, relevant knowledge explored by $\pi_{\theta_{s}}$ in $\{o_i\}_{i=1}^G$ but not yet reflected in $\mathcal{R}_{s_0}$, and \textbf{negative rubrics} that summarize common undesirable behaviors such as reward hacking observed across $\{o_i\}_{i=1}^G$. By observing the responses, the judge continuously tracks and uncovers weaknesses in both the rubric set and solver.

\textbf{Proposer Evolution}. 
The objective of proposer $\pi_{\theta_{p}}$ evolution is to enhance the capability to search for materials and propose research tasks that can expose the weaknesses of the solver, based on the judge's assessment. Similar to solver evolution, HOTE uses GRPO to achieve the evolution of the proposer as shown in Equation~\ref{eq:grpo}, with the distinction that $s_0$ becomes \textbf{proposing research tasks based on the judge's assessment $\mathcal{A}$} represented by $s_0^{p}$, and $\{o_i\}_{i=1}^G$ becomes \textbf{a group of research tasks proposed by the proposer} $\{o_i^{p}\}_{i=1}^{G'}$. However, there are two key issues left:
\begin{itemize}
\item The assessment from judge $\mathcal{A=}\{\text{Judge}_{\pi_{\theta_{j}}}(o_i,R)\mid (R,w)\in \mathcal{R}_{s_0},1\le i\le G\}$ includes rewards for each rubric of every response, thus using all of them as input to the proposer results in excessive length, slowing down training speed.
\item $\{o_i^{p}\}_{i=1}^{G'}$ proposed by the proposer lacks shared rubrics, making it difficult to evaluate their relative strengths and weaknesses. 
\end{itemize}
Therefore, we propose \textbf{meta rubrics}, allowing the judge to summarize assessments into multiple meta rubrics, uncovering common model weaknesses among the solver’s responses as follows:

\begin{align}
\mathcal{R}^{\text{meta}}_{\{o_i^{p}\}_{i=1}^{G'}}=\text{Meta}_{\pi_{\theta_{j}}}(\mathcal{A},\mathcal{R}_{s_0},\{o_i\}_{i=1}^G).
\label{eq:meta}
\end{align}

These meta rubrics serve as the proposer input and persistent rubrics shared across $\{o_i^{p}\}_{i=1}^{G'}$. On one hand, they are used for solver evolution; on the other hand, they leverage the reward of solver responses to compute the reward $r_i^p$ of research task $o_i^p$ as follows:

\begin{equation}
\begin{aligned}
&r_i^p=\frac{1}{M}\textstyle\sum_{(R,w)\in \mathcal{R}^{\text{meta}}_{\{o_i^{p}\}_{i=1}^{G'}}}\\
&\mathbb{I}\cdot(1-\mathbb{E}_{\{o_j\}_{j=1}^G \sim \pi_{\theta_s}(\cdot \mid o_i^p)}[\text{Judge}_{\pi_{\theta_{j}}}(o_j,R)])
\end{aligned}
\label{eq:judge_reward}
\end{equation}
where $M=\left|\mathcal{R}^{\text{meta}}_{\{o_i^{p}\}_{i=1}^{G'}}\right|$. $\mathbb{I}$ represents whether there is an $o_j$ that passes the rubric $R$. `1' represents the max average reward the solver $\pi_{\theta_s}$ can obtain given the judge's reward for each rubric is limited to $0$ or $\pm1$. Through Equation~\ref{eq:judge_reward}, we encourage the proposer to generate challenging but solvable research tasks for the solver.

\subsection{Dual-mode Hybrid Training Strategy} 
The independent evolution of the three modules mentioned above is insufficient. They should complement one another to form an evolution pipeline where a stronger solver stimulates a more refined judge, and a stronger proposer and judge inspire the proposer to formulate more challenging problems, which in turn train a stronger solver.  
As shown in Figure~\ref{fig:method}, our proposed dual-mode hybrid training strategy primarily encompasses three key features.  

\textbf{Hybrid Data}. Except for the first step, the training data for each step (comprising a batch size of $B$ research tasks) consists of $\frac{B}{2}$ research tasks from the original training set and $\frac{B}{2}$ synthetic research tasks proposed by the proposer based on evaluations from the previous step. Beyond leveraging existing data resources and facilitating agent evolution, this design allows synthetic research tasks generated by the proposer to be immediately solved by the solver and evaluated by the judge. The evaluation results can then be used to optimize both the proposer and the judge simultaneously, avoiding the need for repeated sampling.  

\textbf{Diverse Proposing}. We found when the proposer generates research tasks based solely on the judge’s evaluation and all research tasks from the previous step, they tend to concentrate on the same topic, which can undermine the balance and diversity of the training data. Therefore, at each step, we prompt the proposer to generate $N$ groups of research problems based on the judge’s evaluation and $N$ distinct combinations of research tasks from the previous step. 

\textbf{Hybrid Modes}. As illustrated in Figure~\ref{fig:reli}(b), we found that for DR Tulu-8B-SFT, DR Tulu-8B-RL and HOTE trained solely in no-tool mode, their performance on Healthbench under no-tool mode exceeds that under tool-use mode. This phenomenon can be attributed to factors such as noise in the search tool and it is acceptable in practical applications to trade evaluation metrics for research reports with clear references. Intuitively, we think it is easier to learn research report generation techniques excluding reference searching and understanding in a no-tool training mode than a tool-use training mode. Meanwhile, as shown in Figure~\ref{fig:reli}(c), for HOTE trained in no-tool mode, its performance on DRB under tool-use mode exhibits a clear pattern of initial improvement followed by decline, suggesting that no-tool training leads the model to rely excessively on parametric knowledge. Therefore, we randomly assign half of the training data in each step to no-tool mode and the other half to tool-use mode (to ensure fairness in judging synthetic research tasks, this assignment is randomized across the $N$ groups), thereby enhancing research report generation techniques and avoiding over-reliance on parameterized knowledge. 

In actual training, we trained $600$ steps using no-tool mode and then trained $700$ steps using hybrid mode. Besides, we theoretically prove that the hybrid mode results in a lower expected maximum generation time in Appendix~\ref{appendix:proof_time_comparison}.

\section{Experiment}
\label{sec:experiment}

Our experiments aim to address five research questions. \textbf{RQ1:} Does HOTE demonstrate stronger capabilities in handling open-ended research tasks with less time overhead? \textbf{RQ2:} Are the three modules indispensable for HOTE evolution? \textbf{RQ3:} Does HOTE facilitate the collaborative progress of dual modes? \textbf{RQ4:} Is HOTE effective with different base models? \textbf{RQ5:} Does HOTE evolve more sustainably? We additionally provide the case study, the effect of judge models, prompts and diverse proposing in Appendix~\ref{sec:case} and~\ref{sec:effect}.

\begin{table*}[t]
\centering
\caption{Performance comparison across long-form deep research benchmarks. HOTE-8B outperforms existing \textit{Open Deep Research Models}, \textit{Open Deep Research}, \textit{RL Methods} and \textit{Evolving Methods}.}
\small
\scalebox{0.85}{ 
\begin{tabular}{lcc|ccccc|c}
\toprule
\multirow{2}{*}{\textbf{Method}}  
& \multirow{2}{*}{\textbf{HealthBench}} 
& \multirow{2}{*}{\textbf{ResearchQA}} 
& \multicolumn{5}{c|}{\textbf{DRB}} 
& \multirow{2}{*}{\textbf{Average}} \\
\cmidrule(lr){4-8}
& & 
& \textbf{Overall} 
& \textbf{Comp} 
& \textbf{Insight} 
& \textbf{Instruction} 
& \textbf{Readability} 
& \\
\midrule

\rowcolor{gray!20} \multicolumn{9}{c}{\textit{Closed Deep Research}} \\
Gemini 3 Pro + Search & 38.0 & 74.3 & 46.3 & 43.4 & 44.9 & 49.8 & 49.0 & 52.9 \\
GPT-5 + Search & 59.5 & 78.2 & 50.7 & 26.7 & 21.3 & 41.0 & 29.4 & 62.8 \\
OpenAI Deep Research  & 53.8 & 79.2 & 46.9 & 46.8 & 45.2 & 49.2 & 47.1 & 60.0\\


\midrule
\rowcolor{gray!20} \multicolumn{9}{c}{\textit{Open Deep Research Models}} \\
Qwen3-8B & 5.9 & 46.3 & 18.2 & 14.3 & 8.7 & 29.5 & 24.4 & 23.5 \\
Qwen3-235B-A22B & 21.3 & 50.7 & 22.5 & 19.1 & 17.3 & 30.6 & 25.1 & 31.5 \\
Search-R1-7B & -0.1 & 27.9 & 9.5 & 5.2 & 2.1 & 18.6 & 16.8 & 12.4 \\
ASearcher-Web-7B & -13.0 & 19.4 & 7.8 & 5.1 & 1.7 & 15.2 & 11.8 & 4.7 \\
WebExplorer-8B & 33.7 & 64.8 & 36.7 &  33.7 & 28.5 & 45.7 & 42.2 & 45.1 \\
WebThinker-32B-DPO & 11.1 & 48.6 & 23.3 & 19.7 & 12.3 & 36.8 & 26.3 & 27.7 \\
Tongyi DeepResearch-30B-A3B & 46.2 & 66.7 & 40.6 & 39.1 & 34.3 & 46.8 & 45.4 & 51.2\\

\midrule
\rowcolor{gray!20} \multicolumn{9}{c}{\textit{Fixed Pipeline Deep Research}} \\
WebThinker QwQ-32B (report) & 36.5 & 72.8 & 37.9 & 36.2 & 32.6 & 43.2 & 42.9 & 49.1 \\
WebThinker-32B-DPO (report) & 39.4 & 74.2 & 40.6 & 39.4 & 35.4 & 46.0 & 43.5 & 51.4 \\
Ai2 ScholarQA-Claude Sonnet (report) & 32.0 & 75.0 & 36.1 & 35.1 & 32.0 & 40.5 & 38.9 & 47.7 \\

\midrule
\rowcolor{gray!20} \multicolumn{9}{c}{\textit{Open Deep Research}} \\
DR Tulu-8B-SFT  & 38.1 & 68.5 & 39.0 & 36.3 & 35.3 & 45.5 & 39.5 & 48.5 \\
DR Tulu-8B-RL\tablefootnote{We use the 1900-step checkpoint of DR Tulu.} & 50.2 & 74.3 & 43.4 & 41.7 & 41.8 & \textbf{48.2} & 41.3 & 56.0 \\

\midrule
\rowcolor{gray!20} \multicolumn{9}{c}{\textit{RL Methods}} \\
GRPO & 49.6 & 73.5 & 43.1 & 40.8 & 42.1 & 46.9 & 42.6 & 55.4 \\
GSPO & 51.0 & 75.1 & 43.6 & 42.5 & 40.9 & 47.3 & 43.7 & 56.6 \\
REINFORCE++ & 50.8 & 74.8 & 43.1 & 41.2 & 42.7 & 46.1 & 42.4 & 56.2 \\

\midrule
\rowcolor{gray!20} \multicolumn{9}{c}{\textit{Evolving Methods}} \\
SPICE-8B  & 50.2 & 73.9 & 42.1 & 40.6 & 40.9 & 46.1 & 40.8 & 55.4 \\
Dr. Zero-8B & 52.1 & 73.2 & 43.7 & 41.5 & 42.1 & 46.5 & 44.7 & 56.3 \\

\midrule
\rowcolor{gray!20} \multicolumn{9}{c}{\textit{Open Evolving Deep Research}} \\
HOTE-8B & \textbf{54.4} & \textbf{76.9} & \textbf{45.9} & \textbf{44.9} & \textbf{45.4} & 47.8 & \textbf{45.8} & \textbf{59.1} \\

\bottomrule
\end{tabular}
}
\label{tab:deep_research_results}
\end{table*}

%
%
%


\begin{table}[h]
\centering
\caption{Average training time per step for baselines, no-tool mode and hybrid mode of HOTE.}
\small
\begin{tabular}{llcc}
\toprule
\multicolumn{2}{c}{\textbf{Method}} & \textbf{Wall-clock (s/step)} & \textbf{GPU (h/step)} \\

\midrule
\multicolumn{2}{c}{DR Tulu} & 1136.8 & 3.8 \\
\multicolumn{2}{c}{GRPO} & 1024.2 & 3.5 \\
\multicolumn{2}{c}{GSPO} & 1011.3 & 3.4 \\
\multicolumn{2}{c}{REINFORCE++} & 979.6 & 3.3 \\
\multicolumn{2}{c}{SPICE} & 1313.7 & 4.0 \\
\multicolumn{2}{c}{Dr. Zero} & 1301.0 & 4.0 \\
\multirow{2}{*}{HOTE} & no-tool & \textbf{382.0} & \textbf{1.5} \\
                      & hybrid  & \textbf{753.3} & \textbf{2.6} \\
\bottomrule
\end{tabular}
\label{tab:cost_results}
\vspace{-2mm}
\end{table}

\subsection{Evaluations}


\textbf{Benchmark}. We evaluated HOTE and baseline models across three long-form, open-ended benchmarks: HealthBench~\cite{arora2025healthbench} for healthcare deep research, ResearchQA~\cite{yifei2025researchqa} for assessing synthesis over scientific literature, the DeepResearchBench~\cite{du2025deepresearch} (DRB) for evaluating general-domain deep research tasks. For DRB, we additionally provide detailed performance across diverse aspects of the responses. DRB includes the following aspects: Comprehensiveness, Insight, Instruction Following, and Readability. 
In Table~\ref{tab:deep_research_results}, we followed~\cite{shao2025dr} by using HealthBench with 1,000 samples and ResearchQA with 776 samples. For other experimental results, we sampled 100 instances each from HealthBench and ResearchQA respectively.  Please refer to Appendix~\ref{sec:detail} for the benchmark details. 

\textbf{Baselines}. We compared four categories of deep researchers: \textit{Open Deep Research Models}, including Qwen3-8B~\cite{yang2025qwen3}, Qwen3-235B-A22B~\cite{yang2025qwen3}, Search-R1-7B~\cite{jin2025search}, ASearcher-Web-7B~\cite{gao2025beyond}, WebExplorer-8B~\cite{liu2025webexplorer}, WebThinker-32B-DPO~\cite{li2025webthinker}, Tongyi DeepResearch-30B-A3B~\cite{team2025tongyi}; \textit{Open Deep Research}, including DR Tulu-8B-SFT and DR Tulu-8B-RL~\cite{shao2025dr}; \textit{RL Method}, including GRPO~\cite{shao2024deepseekmath}, GSPO~\cite{zheng2025group} and REINFORCE++~\cite{hu2025reinforce++}; \textit{Evolving Method}, including SPICE~\cite{liu2025spice} and Dr. Zero~\cite{yue2026dr}. We also provided \textit{Closed Deep Research}, including Gemini 3 Pro, GPT-5, and OpenAI Deep Research; \textit{Fixed Pipeline Deep Research}, including WebThinker QwQ-32B, WebThinker-32B-DPO, and Ai2 ScholarQA-Claude Sonnet~\cite{singh2025ai2} for reference. Please refer to Appendix~\ref{sec:detail} for implementation details.

\subsection{Training Details}
We utilized Qwen3-8B to initialize the checkpoint of the proposer and DR Tulu-8B-SFT to initialize the checkpoint of the solver. For Open Deep Research, RL Methods, Evolving Methods and Open Evolving Deep Research, we used the same original RL training set in DR Tulu~\cite{shao2025dr} licensed under ODC-BY with 9K samples to ensure fairness. We employed Qwen3-235B-A22B-Instruct-FP8 as the judge. The batch size $B$ was set to $48$, the group size of solver $G$ to $8$, the group size of proposer $G'$ to $6$, the learning rate to $5e-7$, the maximum number of tool uses per response $T$ to $10$, the temperature to $1$, and the response length to $16384$. For the performance in Table~\ref{tab:deep_research_results}, RL Methods and Evolving Methods were trained for 1300 steps until they converged on a held-out validation set while HOTE-8B was trained in no-tool mode for 600 steps and hybrid mode for 700 steps. We provide hyperparameter analysis in Appendix~\ref{sec:hyper}.

\subsection{RQ1: Does HOTE demonstrate stronger capabilities in handling open-ended research tasks with less time overhead?}
As shown in Table~\ref{tab:deep_research_results}, the HOTE-8B model surpasses the open-source solution DR Tulu on HealthBench, ResearchQA and DRB. It also outperforms Open Deep Research Models including Tongyi DeepResearch-30B-A3B. As illustrated in Figure~\ref{fig:abl}, HOTE also leads existing rl and agent evolution methods. Additionally, due to the presence of hybrid mode, only half of the research tasks in the latter $700$ steps out of the total $1300$ steps require tool-use. Given that the maximum number of tool-use per response is $T$, the batch size is $B$ and the group size is $G$, the maximum number of tool-use required for HOTE training is $350BTG+175B$(term 1 for solver, term 2 for proposer). In contrast, for DR Tulu it is $1900BTG$ while for RL Methods and Evolving Methods it is $1300BTG$. Furthermore, as indicated in Table~\ref{tab:cost_results}, even with the addition of proposer evolution, both the no-tool mode in the first $600$ steps and the hybrid mode in the latter $700$ steps contribute to improvements in training speed.


\subsection{RQ2: Are the three modules indispensable for HOTE evolution?}
We compared HOTE, SPICE, the HOTE version without judge evolution (HOTE w/o je, equivalent to Dr. Zero using rubric-based reward and GRPO), and the HOTE version without proposer evolution (HOTE w/o pe, the proposer's parameters are fixed) in the no-tool mode. As shown in Figure~\ref{fig:abl}, when training in the no-tool mode using HOTE, although HOTE initially performed slightly worse on the benchmark compared to SPICE, HOTE w/o je and HOTE w/o pe, it gradually achieved overall superiority as training progressed. More importantly, while HOTE w/o je, HOTE w/o pe and SPICE approached convergence, HOTE maintained a stronger upward trend. Moreover, as can be seen from Figure~\ref{fig:npe}, with proposer evolution enabled, the scores of synthetic research tasks are more stable compared to fixed proposer parameters, which indicates that proposer evolution helps maintain the difficulty of research tasks.


\subsection{RQ3: Does HOTE facilitate the collaborative progress of dual modes?}
Figure~\ref{fig:reli}(a) shows the performance of HOTE, the open-source training approach DR Tulu, as well as HOTE trained exclusively in tool-use mode, evaluated on HealthBench, ResearchQA, and DRB under both no-tool and tool-use modes. HOTE outperforms both DR Tulu and the single-mode version across both no-tool and tool-use evaluation modes, achieving collaborative progress in the dual modes by enhancing research report generation techniques while avoiding over-reliance on parameterized knowledge. 


\begin{figure*}[t]
\centering
\includegraphics[width=0.99\textwidth]{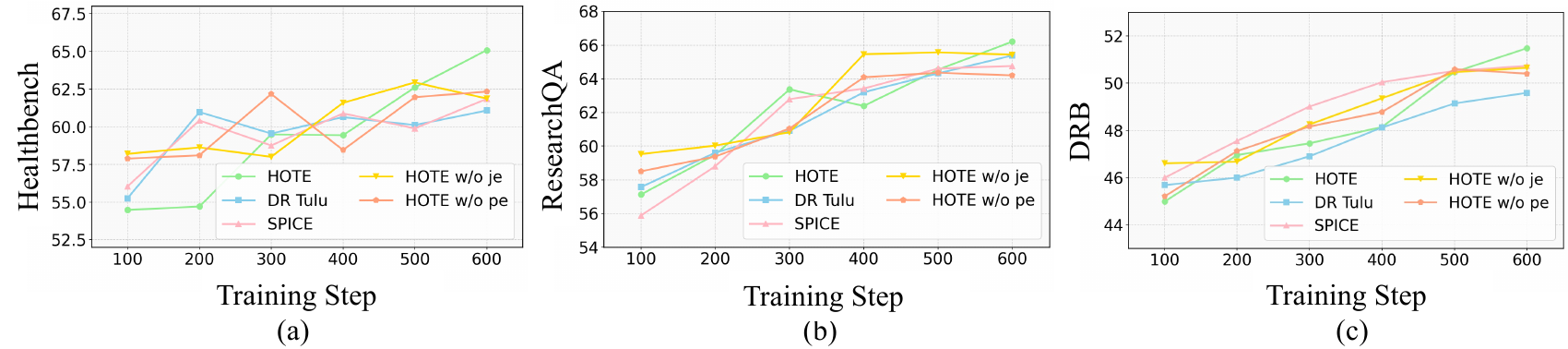}
\caption{In HealthBench (a), ResearchQA (b), and DeepResearchBench (c), after 600 steps of training in no-tool mode, HOTE outperforms SPICE, HOTE w/o je and HOTE w/o pe while demonstrating an upward trend.}
\label{fig:abl}
\end{figure*}

\begin{figure}[t]
\centering
\includegraphics[width=0.49\textwidth]{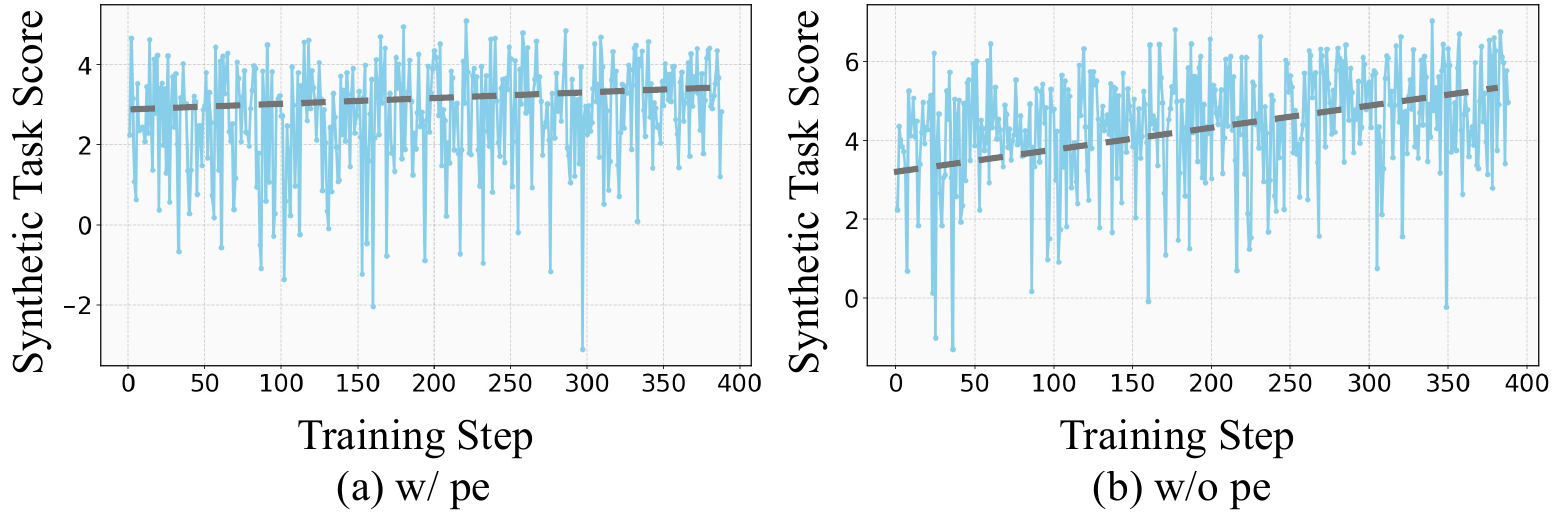}
\caption{(a) During the training in no-tool mode with proposer evolution enabled, the solver's synthetic task score remains stable, indicating that proposer evolution maintains the challenge of the tasks for the evolving solver; (b) After disabling proposer evolution, the solver's synthetic task score gradually increases.}
\label{fig:npe}
\end{figure}


\subsection{RQ4: Is HOTE effective with different base models?}

As shown in Table~\ref{tab:other_models}, we also provided the performance comparison on Llama3.1-8B-Instruct supervised fine-tuned by dr-tulu-sft-data~\cite{shao2025dr} for 5 epochs. HOTE maintains its lead across three benchmarks and over baselines that we can train by ourselves including \textit{Open Deep Research}, \textit{RL Methods} and \textit{Evolving Methods}. The absolute scores are lower than using DR Tulu-8B-SFT fine-tuned from Qwen3-8B due to the lower capability of the base model.

\begin{table}[t]
\centering
\caption{Performance comparison with Llama3.1-8B-Instruct supervised fine-tuned by dr-tulu-sft-data as the base model.}
\small
\scalebox{0.95}{ 
\begin{tabular}{lccc}
\toprule
\textbf{Method} & \textbf{HealthBench} & \textbf{ResearchQA} & \textbf{DRB} \\

\midrule
DR Tulu & 28.8 & 60.0 & 32.3 \\
GRPO & 27.2 & 59.8 & 31.0 \\
GSPO & 27.8 & 59.0 & 31.5 \\
REINFORCE++ & 27.2 & 59.3 & 32.6 \\
SPICE & 29.2 & 62.0 & 32.9 \\
Dr. Zero & 28.4 & 61.0 & 31.5 \\
HOTE & \textbf{33.1} & \textbf{64.2} & \textbf{35.1} \\
\bottomrule
\end{tabular}
}
\label{tab:other_models}
\vspace{-2mm}
\end{table}

\subsection{RQ5: Does HOTE evolve more sustainably than baselines?}

We compared the performance of HOTE and the baselines during training from 1200 to 1500 total steps. We use the average performance on the three benchmarks. As shown in Table~\ref{tab:sustainably}, the baselines have already converged, whereas HOTE not only outperforms the baselines but also continues to exhibit an upward trend. HOTE can sustain continuous evolution for at least 252 hours (1500 steps) of wall-clock time.

\begin{table}[t]
\centering
\caption{Average performance comparison across three benchmarks between HOTE and baselines from 1200 steps to 1500 steps in total. HOTE evolves more sustainably than baselines.}
\small
\scalebox{0.98}{ 
\begin{tabular}{lcccc}
\toprule
\textbf{Method} & \textbf{1200} & \textbf{1300} & \textbf{1400} & \textbf{1500} \\

\midrule
GRPO & 55.4 & 55.4 & 55.2 & 55.2\\
GSPO & 56.4 & 56.6 & 56.6 & 56.4 \\
REINFORCE++ & 56.2 & 56.2 & 55.7 & 56.0 \\
SPICE & 55.5 & 55.4 & 55.2 & 55.4 \\
Dr. Zero & 56.3 & 56.3 & 56.5 & 56.1 \\
HOTE & \textbf{58.0} & \textbf{59.1} & \textbf{59.6} & \textbf{59.9} \\
\bottomrule
\end{tabular}
}
\label{tab:sustainably}
\vspace{-2mm}
\end{table}



\section{Conclusion}    

We propose Hybrid Open-Ended Tri-Evolution (HOTE), aiming to develop a deep researcher capable of autonomous evolution in open-ended environments for open-ended tasks with less time overhead. Through a well-designed reinforcement learning with hybrid modes, HOTE achieves synergistic evolution among the proposer, solver and judge as well as the mutual benefit between no-tool and tool-use modes. On three long-form deep research benchmarks, HOTE-8B outperforms the strongest open 8-32B models and state-of-the-art deep research training methods with less time overhead. In future work, we will continue to explore how to handle noise in real-world search tools during the evolutionary process, how to break free from dependence on original training dataset and how to scale HOTE to larger MoE models.

\section*{Limitations}

The evolution gradually slows down as training progresses and is difficult to obtain perfect scores, suggesting that the upper bound of evolution may still be constrained by model scale. Investigating the scaling capability of HOTE will be a major direction of our future work. The proposed method still relies on the initial training data, but we believe that transcending the limitations of existing training data through evolution is inherently valuable.




\bibliography{acl_latex}

\appendix

\clearpage

\section{Related Work}
\label{sec:related}

We summarize the contribution of HOTE in Table~\ref{tab:method_comparison}.

\begin{table*}[t]
\centering
\small
\caption{The contribution of HOTE.}
\label{tab:method_comparison}
\begin{tabular}{lccccc}
\toprule
\textbf{Method} &
\textbf{Proposer} &
\textbf{Solver} &
\textbf{Judge} &
\textbf{Open-ended} &
\textbf{Open-ended} \\
&
\textbf{Evolve} &
\textbf{Evolve} &
\textbf{Evolve} &
\textbf{Task} &
\textbf{Environment} \\
\midrule
Dr. zero~\cite{yue2026dr} (Meta)   & \cmark & \cmark & \xmark & \xmark & \cmark \\
DR Tulu~\cite{shao2025dr} (Ai2)     & \xmark & \cmark & \cmark & \cmark & \xmark \\
Spice~\cite{liu2025spice} (ICLR 2026) & \cmark & \cmark & \xmark & \xmark & \xmark \\
R-zero~\cite{huang2025r} (ICLR 2026)& \cmark & \cmark & \xmark & \xmark & \xmark \\
\midrule
\textbf{HOTE}     & \textbf{\cmark} & \textbf{\cmark} & \textbf{\cmark} & \textbf{\cmark} & \textbf{\cmark} \\
\bottomrule
\end{tabular}
\end{table*}

\subsection{Deep Research Agents}

Deep research, defined as AI agents' capability to handle open-ended, long-term, and highly complex information retrieval and integration, has become key for AI agents to move beyond conversational interaction toward general autonomy~\cite{hu2025step,openai2025deepresearch}. 

On the inference front, \citet{wu2025agentic,qin2025flash,schmidgall2025agent} have shown that constructing complex workflows and context management can lead to substantial performance improvements. However, such methods rely on manual prompting, lack generality and flexibility, and make it difficult to evaluate the inherent autonomous agent capabilities of the model~\cite{li2025websailor}. On the training front, research has primarily focused on how to end-to-end train autonomous deep research agents based on flexible reasoning paradigms similar to ReAct~\cite{yao2022react}, enabling them to self-plan, acquire knowledge and summarize. Search-R1~\cite{jin2025search} applies reinforcement learning with verifiable rewards (RLVR) to enhance search capabilities and is trained mainly on short-form question answering~\cite{wei2024measuring,wu2025webwalker,ho2020constructing}. This approach has been explored in many recent follow-up studies, including WebExplorer~\cite{liu2025webexplorer}, Tongyi Deep Research~\cite{team2025tongyi} and WebSailor-V2~\cite{li2025websailor}. WebThinker~\cite{li2025webthinker} and MiroThinker~\cite{team2025mirothinker} extend training to longer report generation and more rounds of tool usage. To address the lack of clearly defined evaluation metrics for long-form deep research responses, DR Tulu~\cite{shao2025dr} proposes Reinforcement Learning via Evolving Rubrics (RLER), which dynamically updates evaluation rubrics based on sampled policy responses. Although the above studies enable agents to autonomously conduct research based on user queries, they lack a process for autonomous exploration and improvement of deep research capabilities. Dr. Zero~\cite{yue2026dr} designs a framework based on search-based proposer–solver self-play, enabling the two to co-evolve without exposure to any training data, but it is limited to short-form and easily verifiable question answering. 

Therefore, we propose the first deep research agent evolution framework that supports open-ended long-form report generation tasks, aiming to achieve both practicality and autonomy simultaneously.

\subsection{Agent Evolving with Self-play}

Agent evolution has long been regarded as a pathway toward achieving artificial general intelligence, signifying the capability of agents to autonomously interact with the environment and continuously learn~\cite{liu2025spice}. Self-play offers a highly promising paradigm for agent evolution, wherein an agent system learns from feedback automatically generated through competition with itself. In the domain of games, self-play has led to achievements such as TD-Gammon’s backgammon mastery~\cite{tesauro1995temporal}, AlphaGo’s superhuman performance in Go~\cite{silver2017mastering}, and CICERO’s capability to understand cooperative strategies~\cite{meta2022human}. In the field of large language models, some approaches enable models to serve dual roles as solver and judge, optimizing strategies without the need for human annotation~\cite{chen2024self,wu2024self,yuan2024self,wan2026inference}. However, such evolution is constrained by the queries in the training set, limiting the model’s ability to autonomously explore new knowledge and skills. By assigning the agent system the roles of both query proposer and solver, significant improvements have been achieved in areas such as mathematics, coding, and general reasoning, surpassing the limitations of the original training set and even demonstrating zero-data effectiveness~\cite{huang2025r,zhao2025absolute,wang2025cure,chen2025self}. To further overcome the inherent limitations of agent capabilities, methods such as SPICE~\cite{liu2025spice} and Dr. Zero~\cite{yue2026dr} provide proposers with large-scale corpora and search engines, facilitating the evolution of agent systems in open-ended environments. However, existing approaches remain confined to verifiable tasks, falling short of addressing the reality of numerous open-ended tasks with ambiguous or undefined boundaries encountered in real-world applications. 

Therefore, we propose an open-ended agent evolution framework tailored to open-ended tasks that are difficult to verify. Through mutual play among proposer, solver and judge, the framework enables collaborative evolution with web-scale knowledge.

\section{Proof of Expected Maximum Generation Time Comparison}
\label{appendix:proof_time_comparison}

We formally derive the inequality between the expected maximum generation time of a \textbf{tool-use} strategy and a \textbf{hybrid-mode} strategy.

\subsection{Problem Setup}

Let $X$ denote the random variable representing the generation time for the \textbf{tool-use} mode, and $Y$ denote the generation time for the \textbf{no-tool} mode. We assume these follow normal distributions with identical variances $\sigma^2$ but distinct means:
\begin{align}
    X &\sim \mathcal{N}(\mu_T, \sigma^2), \\
    Y &\sim \mathcal{N}(\mu_N, \sigma^2),
\end{align}
where $\mu_T > \mu_N$. Let $F_X(t)$ and $F_Y(t)$ denote the cumulative distribution functions (CDFs) of $X$ and $Y$, respectively. Since $\mu_T > \mu_N$ and the variances are equal, we have the strict inequality for the CDFs:
\begin{equation}
    F_X(t) < F_Y(t), \quad \forall t \in \mathbb{R}.
    \label{eq:cdf_inequality}
\end{equation}

We consider the number of generations of the solver as $K$.
\begin{itemize}
    \item \textbf{Strategy A (tool-use):} The maximum generation time $M_A$ is defined as the maximum of $K$ independent and identically distributed (i.i.d.) variables $X_1, \dots, X_K \sim X$:
    \begin{equation}
        M_A = \max \{X_1, \dots, X_K\}.
    \end{equation}
    \item \textbf{Strategy B (hybrid mode):} The maximum generation time $M_B$ is defined as the maximum of $K/2$ variables of type $X$ and $K/2$ variables of type $Y$, all mutually independent:
    \begin{equation}
        M_B = \max \{X_1, \dots, X_{K/2}, Y_1, \dots, Y_{K/2}\}.
    \end{equation}
\end{itemize}

\begin{theorem}
The expected maximum generation time of Strategy A is strictly greater than that of Strategy B, i.e., $E[M_A] > E[M_B]$.
\end{theorem}

\begin{proof}
First, we derive the cumulative distribution functions for the random variables $M_A$ and $M_B$. For any $t \in \mathbb{R}$, the probability that the maximum of a set of independent variables is less than or equal to $t$ is the product of their individual probabilities.

For Strategy A:
\begin{equation}
    P(M_A \le t) = \prod_{i=1}^{K} P(X_i \le t) = [F_X(t)]^K.
\end{equation}

For Strategy B:
\begin{align}
    &P(M_B \le t) = \left( \prod_{i=1}^{K/2} P(X_i \le t) \right) \cdot \nonumber  \\
    &\left( \prod_{j=1}^{K/2} P(Y_j \le t) \right) = [F_X(t)]^{K/2} [F_Y(t)]^{K/2}.
\end{align}

Using the inequality from Eq.~\eqref{eq:cdf_inequality}, where $F_X(t) < F_Y(t)$ for all $t$, and noting that $F_X(t) > 0$ for sufficiently large $t$, we compare the two probabilities:
\begin{align}
    P(M_B \le t) &= [F_X(t)]^{K/2} [F_Y(t)]^{K/2} \nonumber \\
    &> [F_X(t)]^{K/2} [F_X(t)]^{K/2} \nonumber \\
    &= [F_X(t)]^K \nonumber \\
    &= P(M_A \le t).
\end{align}
Thus, $P(M_B \le t) > P(M_A \le t)$ for all $t$ where $F_X(t) > 0$. This implies that $M_A$ stochastically dominates $M_B$ (first-order stochastic dominance). 

In terms of the survival function (tail probability), this inequality is reversed:
\begin{align}
    P(M_A > t) &= 1 - P(M_A \le t) \\
    &> 1 - P(M_B \le t) \\
    &= P(M_B > t).
\end{align}

The expected value of a random variable $Z$ can be expressed as the integral of its survival function over its support. Assuming the support covers the real line:
\begin{align}
    E[Z] &= \int_{-\infty}^{\infty} t f_Z(t) dt \\
    &= \int_{0}^{\infty} P(Z > t) dt - \int_{-\infty}^{0} P(Z \le t) dt.
\end{align}
Given the stochastic dominance established above, the strict inequality holds for the expectation:
\begin{equation}
    E[M_A] > E[M_B].
\end{equation}
\end{proof}

\section{Case Study}
\label{sec:case}
We conducted case studies on HOTE-8B and DR Tulu-8B to illustrate the advantages of HOTE. We omit the \texttt{think} and \texttt{tool} because of they are too long. In practical applications, the final research report will be additionally appended with the searched references. In Case 1, as shown in Figure~\ref{fig:case1_1}-\ref{fig:case1_2}, HOTE demonstrates: (a) more comprehensive information: the response from HOTE-8B provides detailed citations from EACS guidelines, including baseline examination items (viral load, CD4 count, complete blood count, metabolic indicators, TB screening, opportunistic infection assessment, etc.); (b) better structure: the answer is clearly organized into sections such as "Summary", "Baseline workup" and "Virologic check points"; (c) stronger contextual awareness: it correctly identifies that this is a question for medical professionals, offering detailed guidelines suitable for their level. In contrast, DR Tulu offers a more concise response, presenting only "Bottom line" recommendations and lacking a complete monitoring timeline and baseline examination details. In Case 2, as shown in Figure~\ref{fig:case2_1}-\ref{fig:case2_5}, HOTE (a) correctly identifies an emergency: clearly states that "acute angle‑closure glaucoma is a true ophthalmic emergency" requiring "immediate evaluation and treatment to prevent rapid, irreversible vision loss"; (b) provides specific action advice: explains what the patient should do (seek evaluation by an ophthalmologist) and what examinations the doctor will perform; (c) offers complete clinical information: including symptom descriptions (severe eye pain, blurred vision, halos, headache, nausea) and treatment methods (laser peripheral iridotomy). DR Tulu, while providing background medical knowledge, fails to clearly inform the patient that this is an emergency requiring immediate medical attention.

\section{Algorithm}
We provide the complete training process of HOTE in Algorithm~\ref{alg:hote_training}.

\section{The effect of judge models for training, prompts and diverse proposing}
\label{sec:effect}
We further trained with Qwen3-235BA22B-Think and Qwen3-30BA3B-Instruct as the judge model (2507-FP8 version), along with the average wall-clock time per step. The results in Table~\ref{tab:judge_model_comparison} show that: (i) a smaller-scale judge model leads to a moderate performance degradation; (ii) a thinking model achieves nearly identical performance but substantially reduces training efficiency. Therefore, we recommend using large-scale open-source instruct models to strike a balance between effectiveness and computational overhead. We consistently set: temperature=0, max\_tokens=16384, top\_p=1.0.

The prompts are role-defining system instructions for the proposer, solver, and judge modules, designed to specify each module’s task and output format. The samples are minimal format demonstrations and are not part of the evaluation benchmarks. To test sensitivity, we replaced samples and rephrased role-defining instructions with three different sets on HOTE and three baselines. We observed negligible impact on final performance in Table~\ref{tab:diff}, suggesting that the method is not materially dependent on a particular prompt/sample choice.

As shown in Table~\ref{tab:diverse}, the diverse proposing effectively improve the performance on three benchmarks, illustrating its importance in ensuring the quality of proposed research tasks.

\begin{table*}[t]
\centering
\small
\caption{Performance and training efficiency under different judge models.}
\label{tab:judge_model_comparison}
\begin{tabular}{lcccc}
\toprule
\textbf{Judge Model} & \textbf{HealthBench} & \textbf{ResearchQA} & \textbf{DRB} & \textbf{Avg. Wall-clock Second/Step} \\
\midrule
Qwen3-235BA22B-Think     & 54.6 & 76.9 & 45.5 & 1065.5 \\
Qwen3-30BA3B-Instruct    & 50.2 & 72.2 & 42.0 & 723.7 \\
Qwen3-235BA22B-Instruct  & 54.4 & 76.9 & 45.9 & 753.3 \\
\bottomrule
\end{tabular}
\end{table*}

\begin{table*}[t]
\centering
\small
\caption{The effect of diverse proposing.}
\label{tab:diverse}
\begin{tabular}{lcc}
\toprule
\textbf{Benchmark} & \textbf{w/ diverse proposing} & \textbf{w/o diverse proposing} \\
\midrule
HealthBench & 54.4 & 50.2 \\
ResearchQA  & 76.9 & 74.1 \\
DRB         & 45.9 & 42.2 \\
\bottomrule
\end{tabular}
\end{table*}

\begin{table*}[t]
\centering
\small
\caption{Performance statistics across three evaluation runs.}
\label{tab:run}
\begin{tabular}{lc}
\toprule
\textbf{Benchmark} & \textbf{Mean$\pm$Std} \\
\midrule
HealthBench & 54.5$\pm$0.1 \\
ResearchQA  & 76.7$\pm$0.2 \\
DRB         & 45.9$\pm$0.0 \\
\bottomrule
\end{tabular}
\end{table*}

\begin{table*}[t]
\centering
\small
\caption{Performance comparison across different samples and role-defining instructions.}
\label{tab:diff}
\begin{tabular}{llccc}
\toprule
\textbf{Exp} & \textbf{Method} & \textbf{HealthBench} & \textbf{ResearchQA} & \textbf{DRB} \\
\midrule
\multirow{4}{*}{Different samples}
& Search-R1-7B                & 0.0$\pm$0.0 & 28.0$\pm$0.2 & 9.5$\pm$0.0 \\
& WebExplorer-8B              & 33.7$\pm$0.1 & 64.8$\pm$0.2 & 36.7$\pm$0.0 \\
& Tongyi DeepResearch-30B-A3B & 46.2$\pm$0.4 & 66.7$\pm$0.1 & 40.6$\pm$0.1 \\
& HOTE-8B            & \textbf{54.4$\pm$0.2} & \textbf{76.6$\pm$0.2} & \textbf{46.3$\pm$0.3} \\
\midrule
\multirow{4}{*}{Different role-defining instructions}
& Search-R1-7B                & 0.0$\pm$0.0 & 27.9$\pm$0.2 & 9.5$\pm$0.0 \\
& WebExplorer-8B              & 33.8$\pm$0.1 & 64.8$\pm$0.1 & 36.7$\pm$0.0 \\
& Tongyi DeepResearch-30B-A3B & 46.2$\pm$0.1 & 66.6$\pm$0.2 & 40.6$\pm$0.0 \\
& HOTE-8B            & \textbf{54.4$\pm$0.1} & \textbf{76.7$\pm$0.2} & \textbf{45.9$\pm$0.1} \\
\bottomrule
\end{tabular}
\end{table*}

\section{Details}
\label{sec:detail}

\subsection{Implementation details}

For \textit{RL Methods} and \textit{Evolving Methods} that we can fully control the training process, since long-form deep research tasks do not have standard reference answers, we consistently adapted them from RLVR to rubric-based reward following~\cite{shao2025dr} without judge evolution. For \textit{Evolving Methods} including SPICE and Dr. Zero, we also consistently adapted them in the same manner as HOTE by utilizing Qwen3-8B to initialize the proposer checkpoint and DR Tulu-8B-SFT to initialize the solver checkpoint. For \textit{Open Deep Research Models}, \textit{Open Deep Research}, \textit{RL Methods}, \textit{Evolving Methods} and HOTE that we can fully control the inference process, we use Serper API for google\_search, Jina API for web\_browse and Semantic Scholar API for paper\_search. We ensured that no data from the benchmark was added to the training set, and we also blocked search tools from accessing the benchmark website. For \textit{Closed Deep Research} that we cannot fully control the training and inference process, we also provide their results for reference following~\cite{shao2025dr} but not for a strict comparison with them. We plan to fully release our models and codes upon acceptance.

\subsection{Benchmark details}

\textbf{Judge}. To avoid the model simply using the biases of the judge during training, and also to follow the official evaluation of HealthBench~\cite{arora2025healthbench}, DRB~\cite{du2025deepresearch} and ResearchQA~\cite{yifei2025researchqa}, different judge models were employed for different benchmarks: GPT-4.1 was used for Healthbench; Gemini-2.5-flash for DRB; and GPT-4.1-mini for ResearchQA. Higher scores consistently indicate better quality across all benchmarks. HealthBench calculates a normalized score based on physician-created rubrics that reward desired behaviors and penalize undesirable ones; ResearchQA measures the thoroughness of addressing literature-derived criteria on a 0-100\% scale; and DRB computes a macro-average score across four quality dimensions via comparison against high-quality reference reports. Please refer to the references for specific implementation.

\textbf{Reliability}. We run the evaluation of HOTE on all three benchmarks three times. As shown in Table~\ref{tab:run}, the standard deviations are small and HOTE consistently maintains its lead. LLM-as-a-judge can provide a stable evaluation for the three benchmarks. Besides, human experts are substantially involved in rubric design across all three benchmarks to ensure the reliability: HealthBench uses conversation-specific rubrics written by 262 physicians, with consensus criteria added only when a majority of reviewing physicians agree they are relevant; ResearchQA derives query-specific rubrics from expert-written survey sections and further validates them with 31 Ph.D. annotators across 8 fields; and DRB builds on tasks crafted and iteratively refined by domain experts, while its adaptive criteria are anchored in four top-level dimensions established from domain expertise: comprehensiveness, insight, instruction-following, and readability.

\textbf{The choice of benchmark}. The chosen benchmarks are for a complementary evaluation in different domains: they evaluate distinct aspects of long-form deep research quality including (Healthbench) health-related safety and communication quality, (Researchqa) scholarly synthesis across 7 research domains (Life \& Earth Sciences, Engineering \& Computer Science, Physical Sciences, Health Sciences \& Medicine, Social Sciences, Humanities, Economics), and (DRB) end-to-end deep-research report quality across 22 domains (Science \& Technology, Finance \& Business, Software Development, Education \& Jobs, Health, Literature, History, Hardware, Industrial, Art \& Design, Games, Crime \& Law, Entertainment, Sports \& Fitness, Software, Transportation, Religion, Home \& Hobbies, Travel, Food \& Dining, Fashion \& Beauty, Social Life) rather than a single narrow criterion.

\section{Hyperparameter analysis}
\label{sec:hyper}
We used HealthBench, ResearchQA and DRB to analyze the impact of different batch sizes $B$, solver group sizes $G$, proposer group sizes $G'$, and the numbers of training steps in no-tool mode and tool-use mode. As shown in Table~\ref{tab:hyperparam}, increasing $B$, $G$, and $G'$ first improves performance and then leads to a plateau. Therefore, we select $B=48$, $G=8$, and $G'=6$. We further conduct the tool-use training until convergence after different steps of no-tool training to explore the effect of no-tool steps. As shown in Table~\ref{tab:hyperparam}, increasing the number of no-tool training steps first improves performance and then causes a decline, possibly because the model becomes overly reliant on parametric knowledge as training progresses. Therefore, we train the no-tool mode for 600 steps. For the learning rate, maximum number of tool uses per response, temperature, and response length, we reused the hyperparameter settings ablated in~\cite{shao2025dr}.

\begin{table*}[t]
\centering
\small
\caption{Hyperparameter analysis on ResearchQA, HealthBench, and DRB.}
\begin{tabular}{cccccccc}
\toprule
$B$ & $G$ & $G'$ & No-tool Steps & Tool-use Steps & ResearchQA & HealthBench & DRB \\
\midrule
24 & 6 & 4 & 600 & 700 & 72.8 & 50.8 & 41.7 \\
40 & 7 & 5 & 600 & 700 & 74.5 & 53.1 & 44.3 \\
48 & 8 & 6 & 600 & 700 & 76.9 & 54.4 & 45.9 \\
64 & 10 & 8 & 600 & 700 & 76.9 & 54.5 & 45.5 \\
48 & 8 & 6 & 400 & 700 & 74.2 & 52.7 & 43.6 \\
48 & 8 & 6 & 800 & 700 & 72.3 & 51.4 & 42.1 \\
\bottomrule
\end{tabular}
\label{tab:hyperparam}
\end{table*}

\section{Specific prompts}
Figure~\ref{fig:prompt1} and Figure~\ref{fig:prompt2} show the system prompts for the solver in tool-use mode.
Figure~\ref{fig:prompt3} shows the system prompt for the solver in no-tool mode.
Figure~\ref{fig:prompt4} and Figure~\ref{fig:prompt5} show the system prompts for the proposer in tool-use mode.
Figure~\ref{fig:prompt6} shows the user prompt for the proposer in tool-use mode.
Figure~\ref{fig:prompt7} and Figure~\ref{fig:prompt8} show the system prompt and user prompt for the proposer in no-tool mode, respectively.
Figure~\ref{fig:prompt9} and Figure~\ref{fig:prompt10} show the system prompts for the judge updating rubrics according to Equation~\ref{eq:judge_evolve}.
Figure~\ref{fig:prompt11} shows the system prompt for the judge assigning rewards based on rubrics.
Figure~\ref{fig:prompt12} and Figure~\ref{fig:prompt13} show the system prompts for the judge generating meta rubrics.

\clearpage
\begin{algorithm*}[t]
\caption{Dual-mode Hybrid Training Strategy for HOTE}
\label{alg:hote_training}
\begin{algorithmic}[1]
\REQUIRE Solver $\pi_{\theta_{s}}$, Proposer $\pi_{\theta_{p}}$, Judge $\pi_{\theta_{j}}$.
\REQUIRE Training dataset $\mathcal{D}_{\text{train}}$.
\REQUIRE Hyperparameters: Batch size $B$, Group size $G$, Number of diverse proposing groups $N$.
\STATE \textbf{Initialize:} Set initial synthetic tasks $\mathcal{D}_{\text{syn}}=\varnothing$.
\WHILE{not converged}
    \STATE \textbf{\textcolor{blue}{\textit{// 1. Hybrid Data Preparation}}}
    \STATE Sample real tasks $\mathcal{D}_{\text{real}}$ of size $B/2$ from $\mathcal{D}_{\text{train}}$.
    \STATE Construct current batch $\mathcal{S} \leftarrow \mathcal{D}_{\text{real}} \cup \mathcal{D}_{\text{syn}}$.
    
    \STATE \textbf{\textcolor{blue}{\textit{// 2. Hybrid Mode Assignment}}}
    \STATE Randomly assign inference mode $m \in \{\texttt{tool-use}, \texttt{no-tool}\}$ to each task in $\mathcal{S}$ ($50\%$ each).
    
    \STATE \textbf{\textcolor{blue}{\textit{// 3. Solver Rollout}}}
    \STATE For each task $s_0 \in \mathcal{S}$, sample $G$ responses $\{o_i\}_{i=1}^G \sim \pi_{\theta_{s}}(\cdot \mid s_0)$ under assigned mode $m$.
    
    \STATE \textbf{\textcolor{blue}{\textit{// 4. Judge Evolution \& Evaluation}}}
    \FOR{each task $s_0 \in \mathcal{S}$}
        \STATE Update active rubrics: $\mathcal{R}^{\text{active}}_{s_0} \leftarrow \text{Update}_{\pi_{\theta_{j}}}(s_0, \{o_i\}_{i=1}^{G}, \mathcal{R}^{\text{active}}_{s_0})$ (Equation~\ref{eq:judge_evolve}).
        \STATE Calculate rewards $r_i$ for each response $o_i$ using $\mathcal{R}_{s_0}$ (Equation~\ref{eq:rubric}).
    \ENDFOR
    \STATE Collect assessments $\mathcal{A}$ containing all rubrics and rewards.
    \STATE Generate meta rubrics $\mathcal{R}^{\text{meta}}$ summarizing weaknesses from $\mathcal{A}$, $\{o_i\}_{i=1}^{G}$ and $\mathcal{R}_{s_0}$ (Equation~\ref{eq:meta}).
    
    \STATE \textbf{\textcolor{blue}{\textit{// 5. Solver Evolution}}}
    \STATE Update solver parameters $\theta_s$ via GRPO (Equation~\ref{eq:grpo}) using rewards $\{r_i\}$.
    
    \STATE \textbf{\textcolor{blue}{\textit{// 6. Proposer Evolution}}}
    \IF{$\mathcal{D}_{\text{syn}} \neq \varnothing$}
    \STATE Calculate proposer rewards $\{r_i^p\}_{i=1}^{G'}$ for tasks in $\mathcal{D}_{\text{syn}}$ (Eq.~\ref{eq:judge_reward}).
    \STATE Update proposer parameters $\theta_p$ via GRPO using rewards $\{r_i^p\}_{i=1}^{G'}$.
    \ENDIF

    \STATE \textbf{\textcolor{blue}{\textit{// 7. Diverse Proposing (Next Step Synthetic Data)}}}
    \STATE Sample $N$ combinations of tasks and corresponding assessments from $\mathcal{S}$.
    \STATE Proposer generates new synthetic tasks $\mathcal{D}_{\text{syn}}'=\{o_i^p\}_{i=1}^{G'}$ conditioned on combinations and $\mathcal{R}^{\text{meta}}$.
    \STATE Update $\mathcal{D}_{\text{syn}} \leftarrow \mathcal{D}_{\text{syn}}'$ for the next iteration.
\ENDWHILE
\end{algorithmic}
\end{algorithm*}

\begin{figure*}[t]
\centering
\includegraphics[width=0.85 \textwidth]{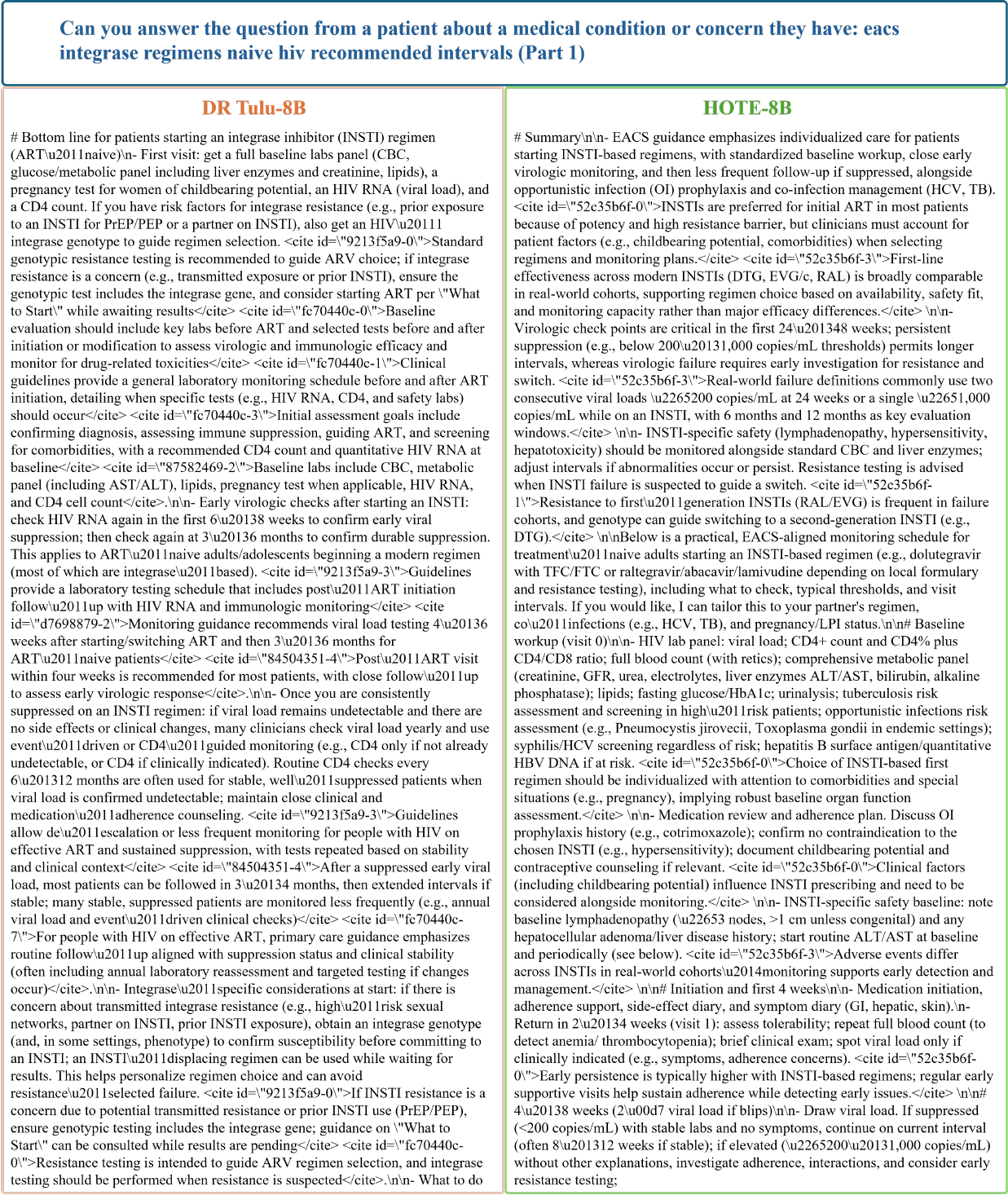}
\caption{Case 1 (Part 1)}
\label{fig:case1_1}
\end{figure*}

\begin{figure*}[t]
\centering
\includegraphics[width=0.85 \textwidth]{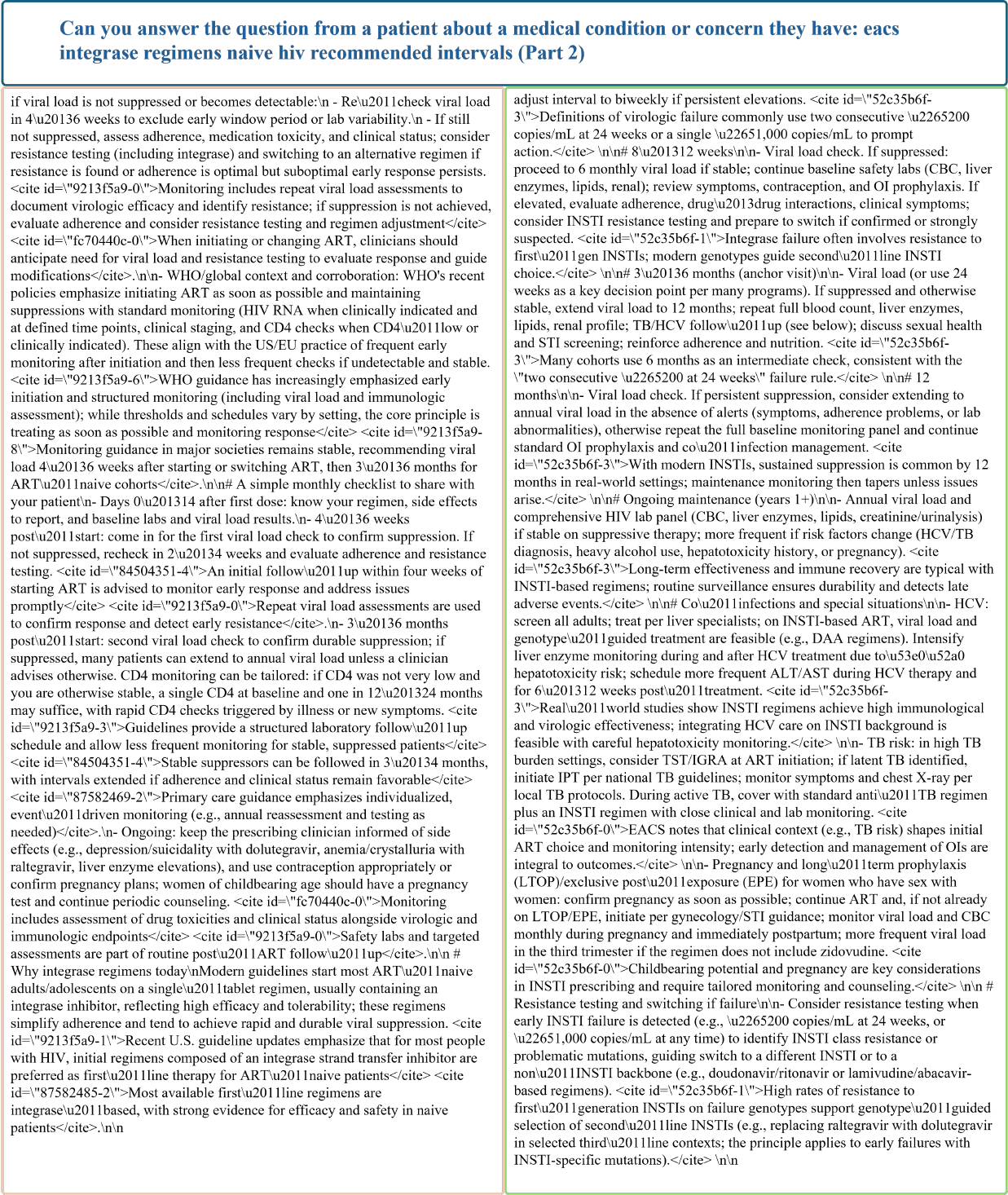}
\caption{Case 1 (Part 2)}
\label{fig:case1_2}
\end{figure*}

\begin{figure*}[t]
\centering
\includegraphics[width=0.85 \textwidth]{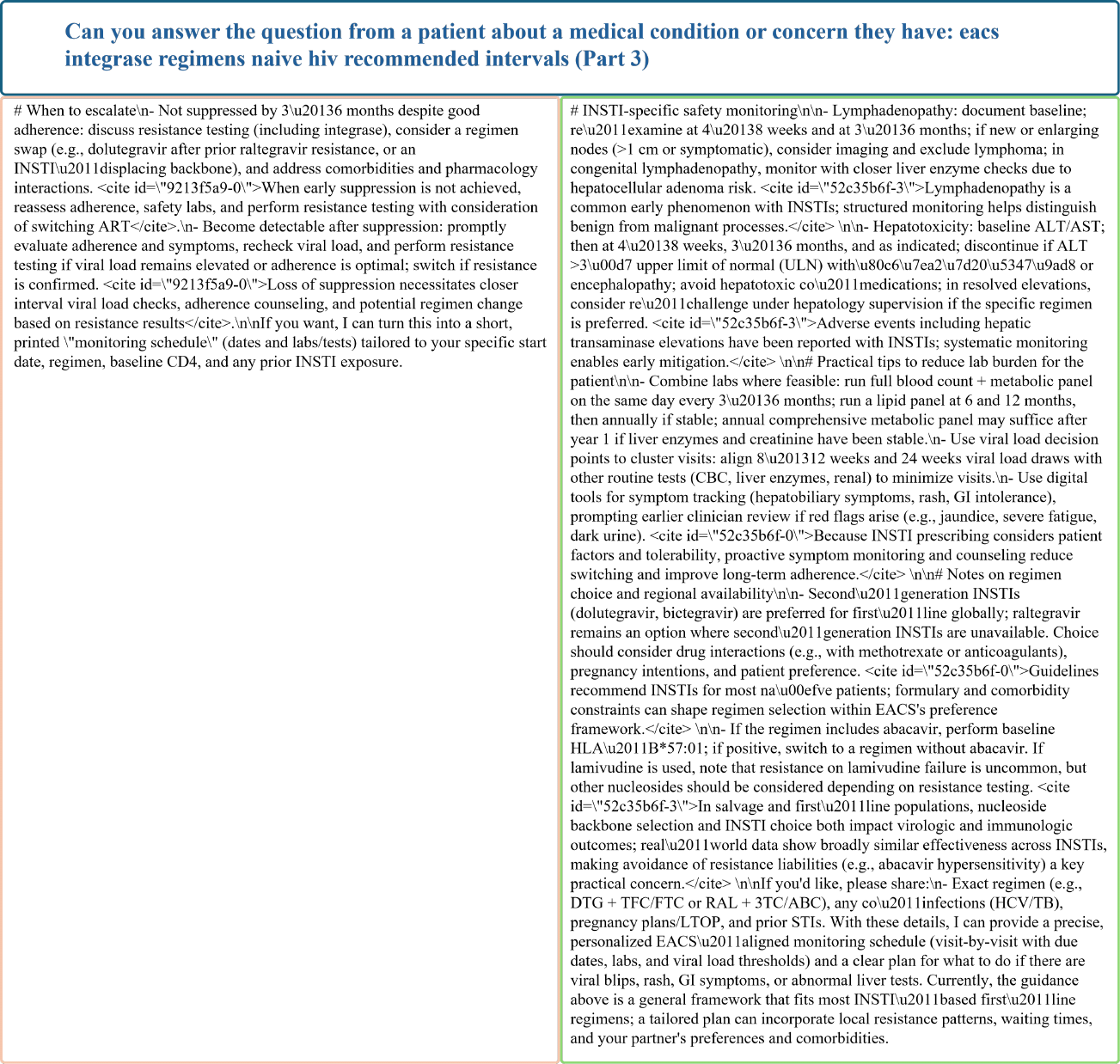}
\caption{Case 1 (Part 3)}
\label{fig:case1_3}
\end{figure*}

\begin{figure*}[t]
\centering
\includegraphics[width=0.85 \textwidth]{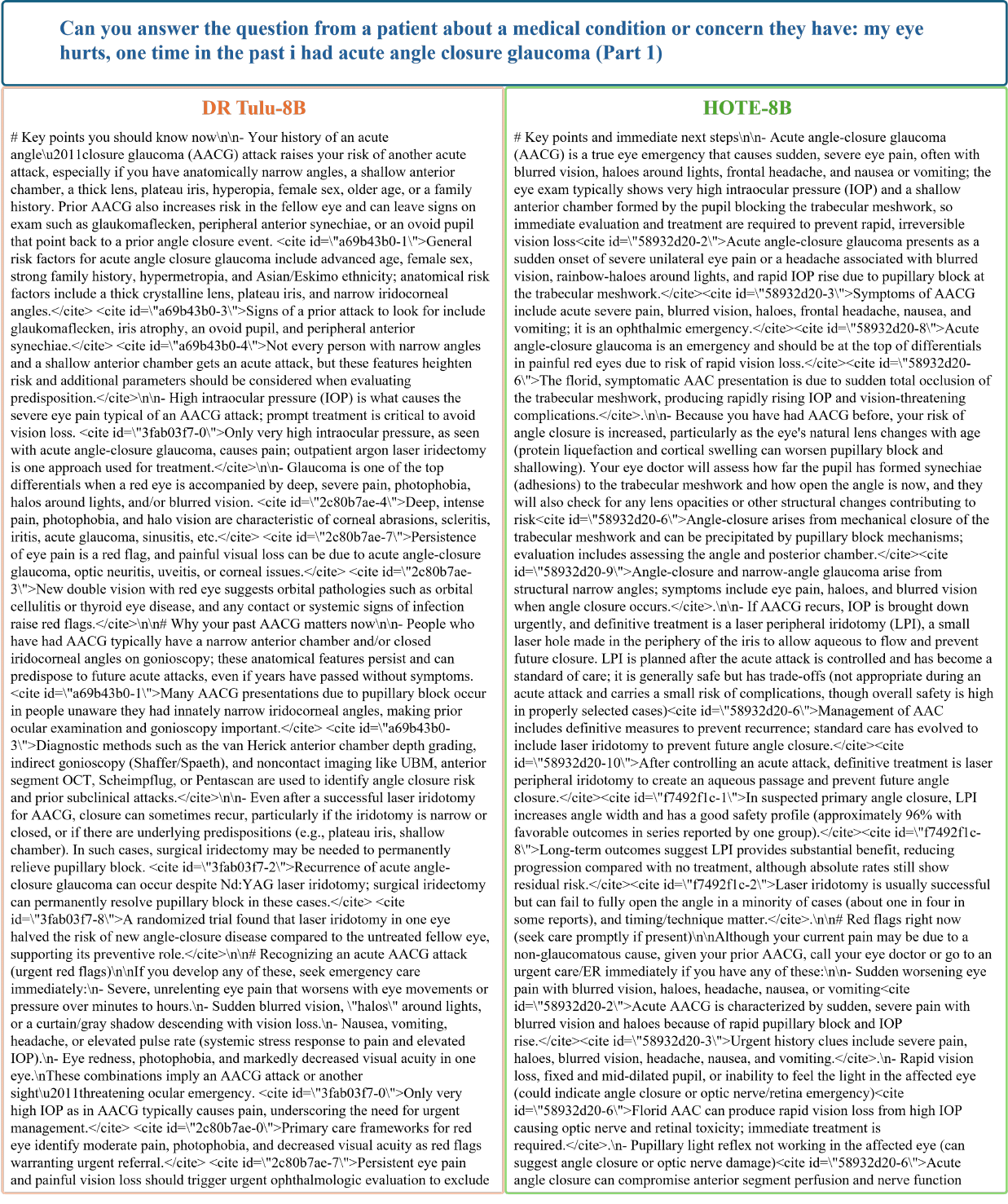}
\caption{Case 2 (Part 1)}
\label{fig:case2_1}
\end{figure*}

\begin{figure*}[t]
\centering
\includegraphics[width=0.85 \textwidth]{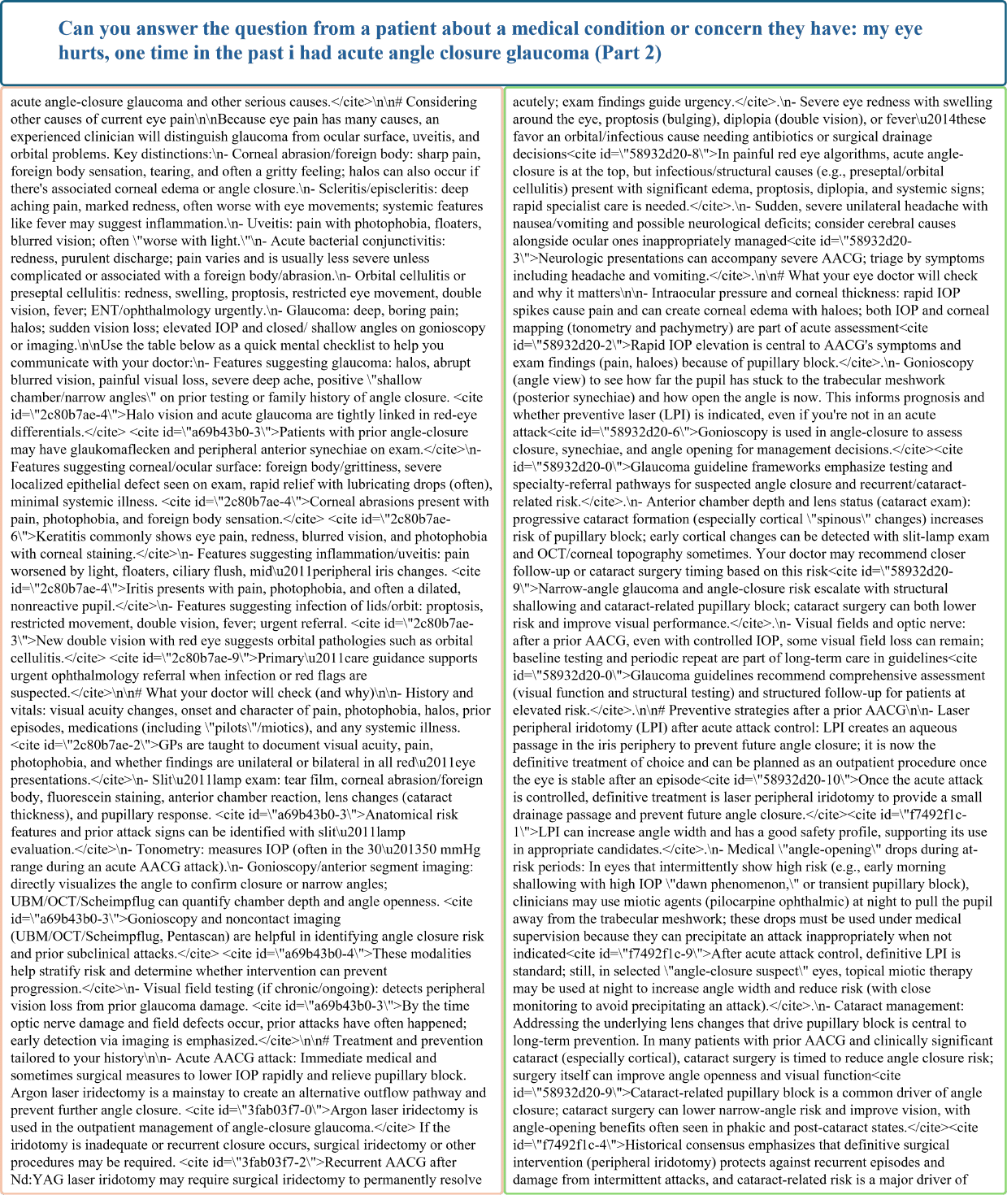}
\caption{Case 2 (Part 2)}
\label{fig:case2_2}
\end{figure*}

\begin{figure*}[t]
\centering
\includegraphics[width=0.85 \textwidth]{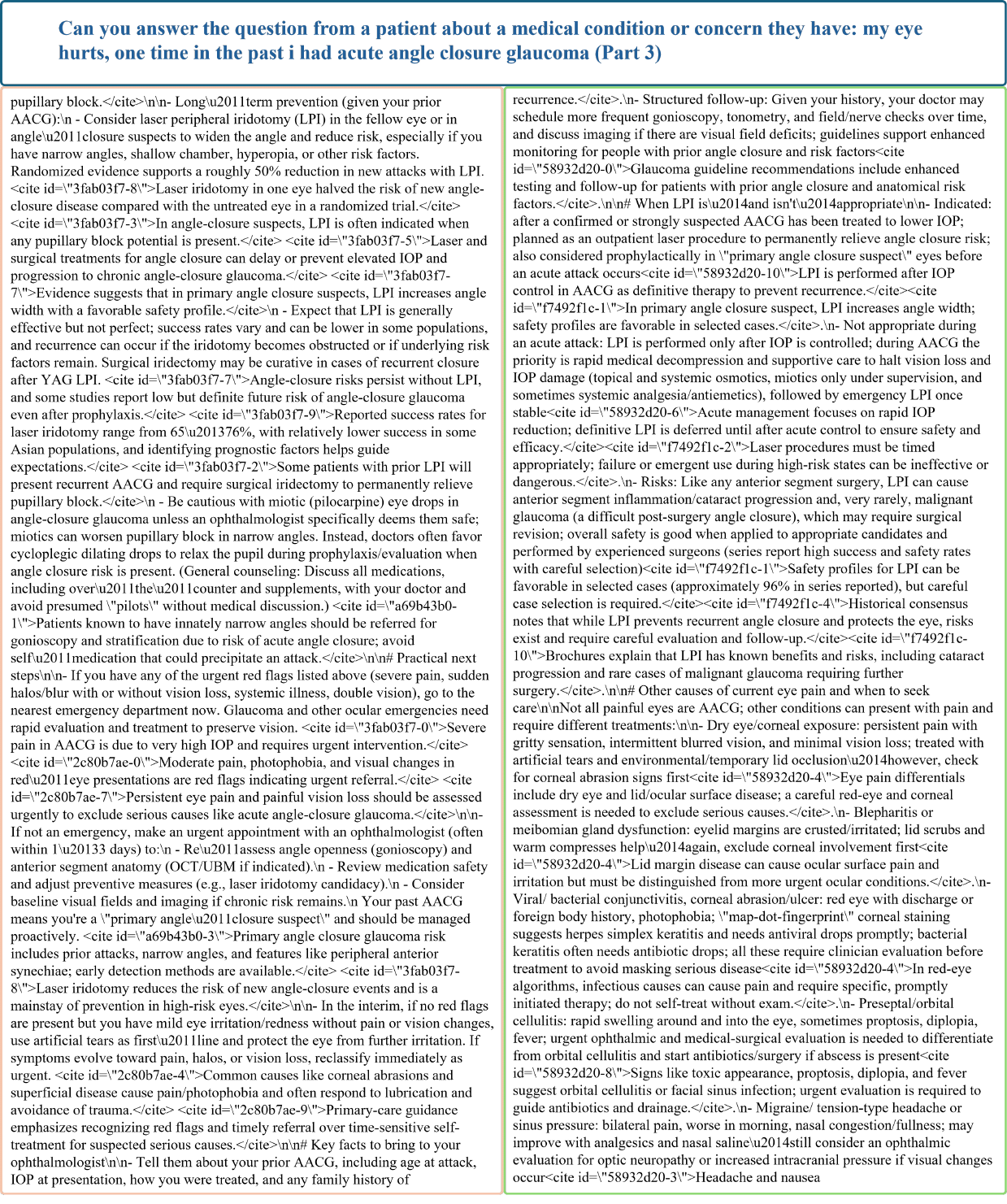}
\caption{Case 2 (Part 3)}
\label{fig:case2_3}
\end{figure*}

\begin{figure*}[t]
\centering
\includegraphics[width=0.85 \textwidth]{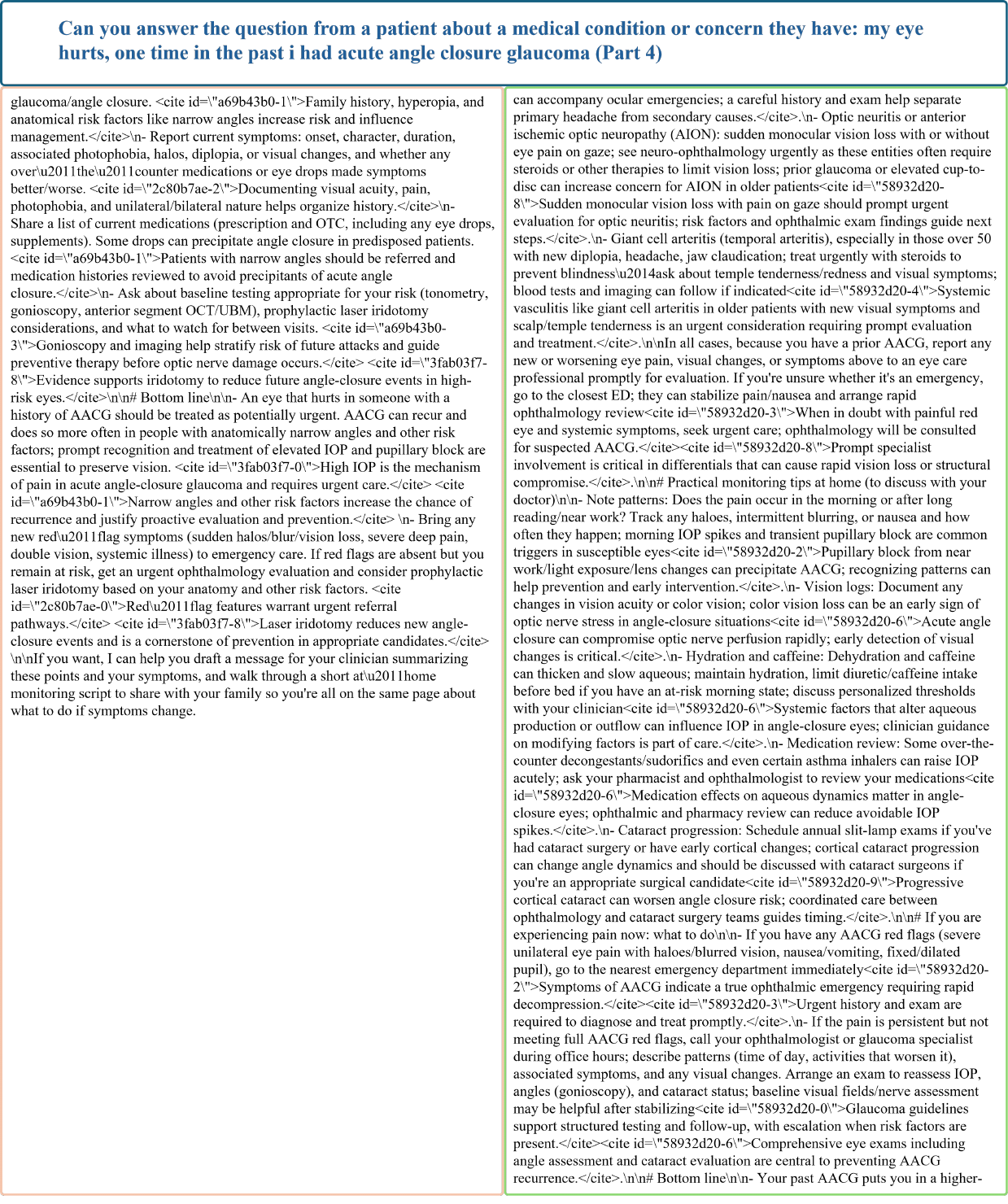}
\caption{Case 2 (Part 4)}
\label{fig:case2_4}
\end{figure*}

\begin{figure*}[t]
\centering
\includegraphics[width=0.85 \textwidth]{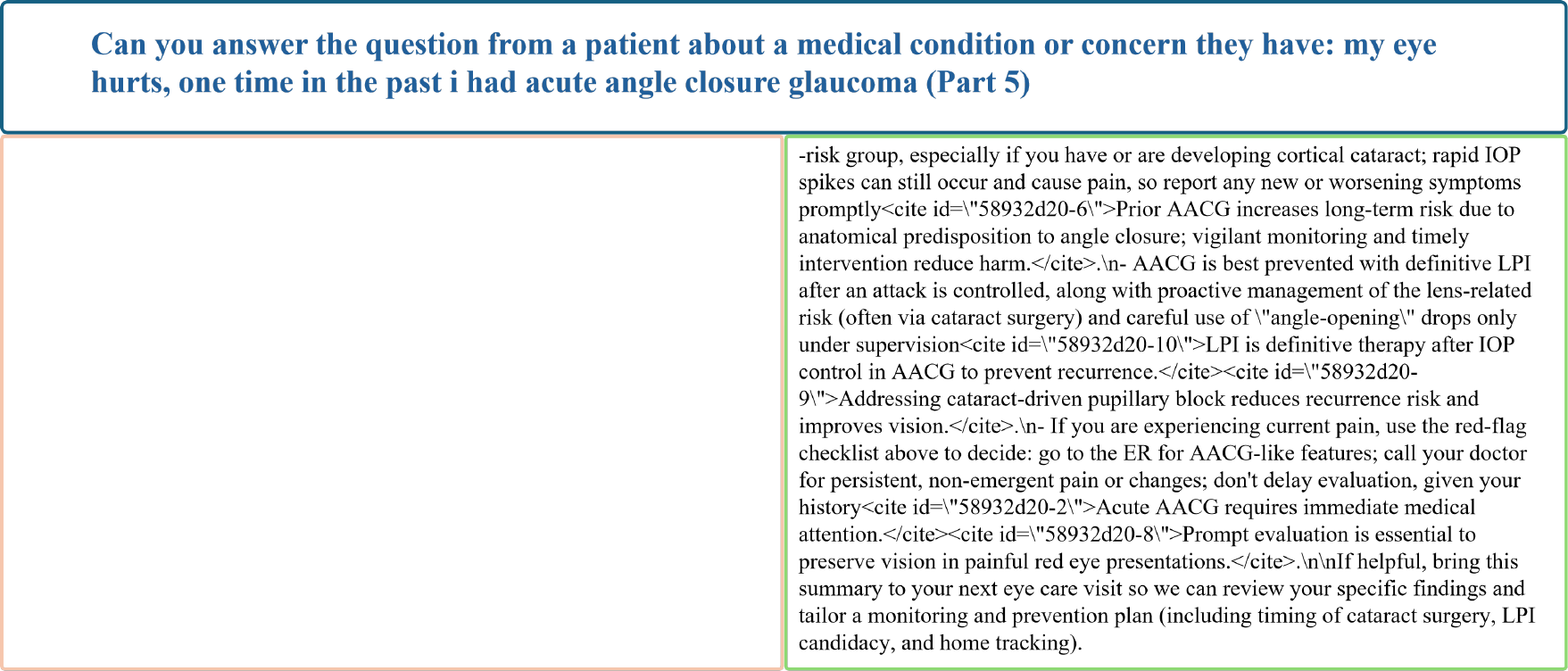}
\caption{Case 2 (Part 5)}
\label{fig:case2_5}
\end{figure*}

\begin{figure*}[t]
\centering
\includegraphics[width=0.74 \textwidth]{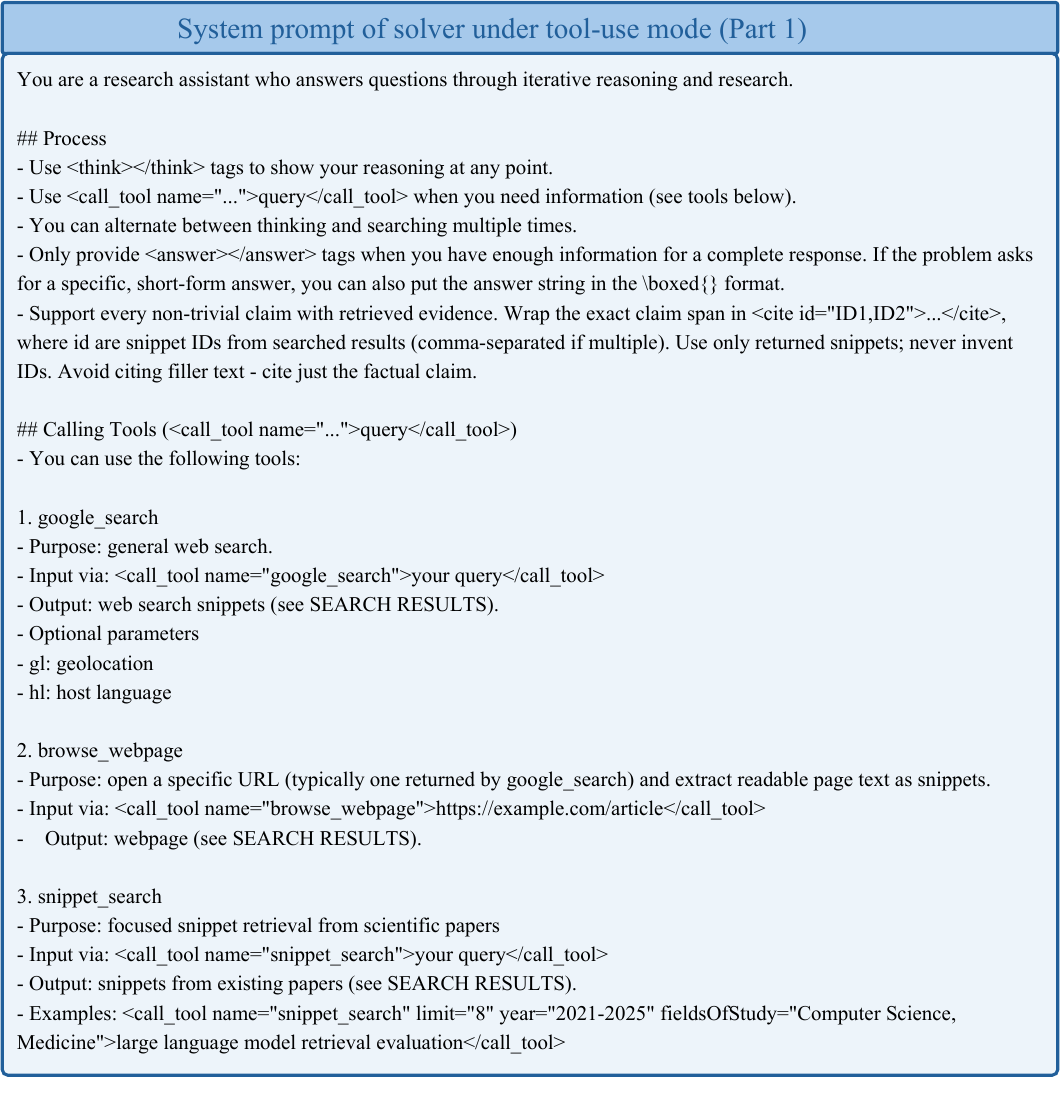}
\caption{System prompt of solver under tool-use mode (Part 1)}
\label{fig:prompt1}
\end{figure*}

\begin{figure*}[t]
\centering
\includegraphics[width=0.74 \textwidth]{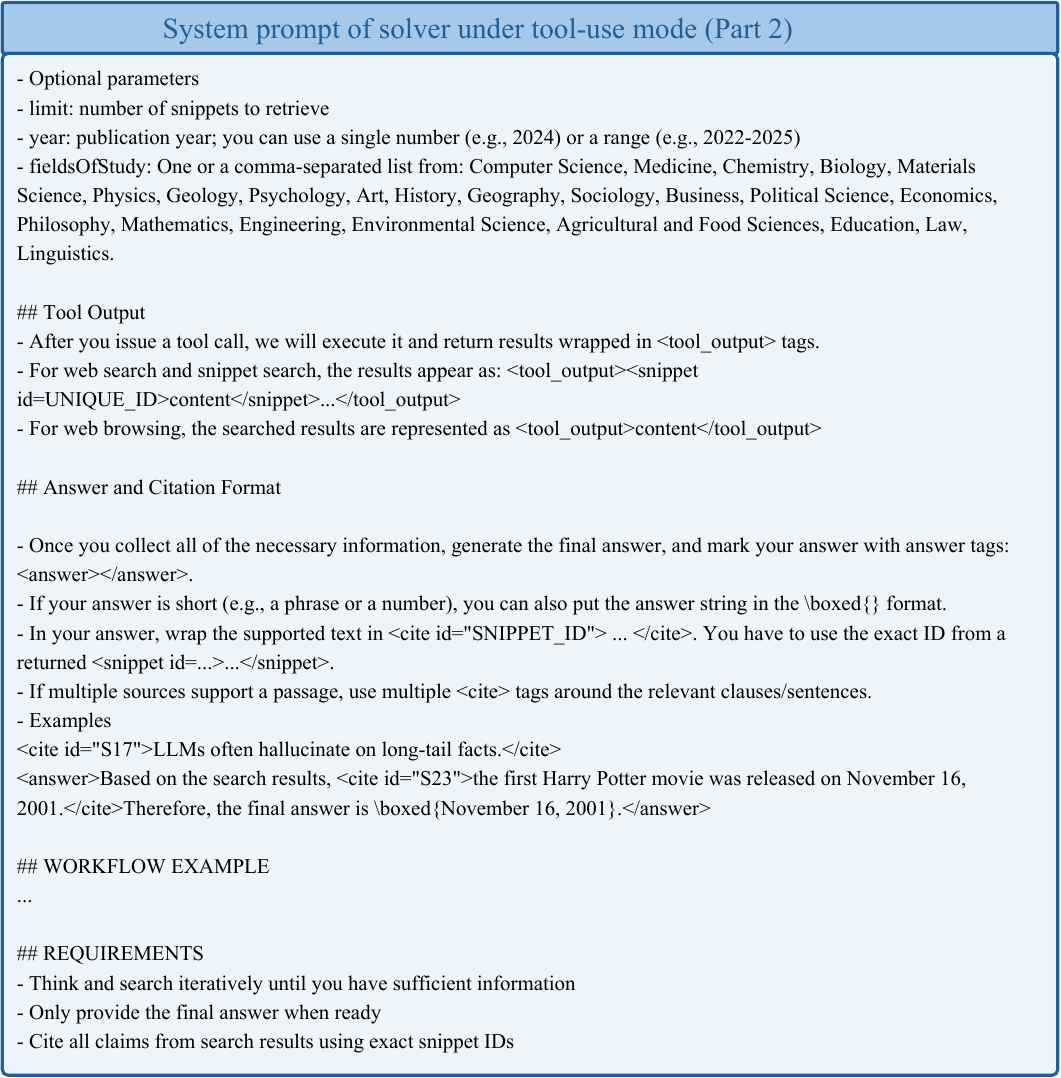}
\caption{System prompt of solver under tool-use mode (Part 2)}
\label{fig:prompt2}
\end{figure*}

\begin{figure*}[t]
\centering
\includegraphics[width=0.74 \textwidth]{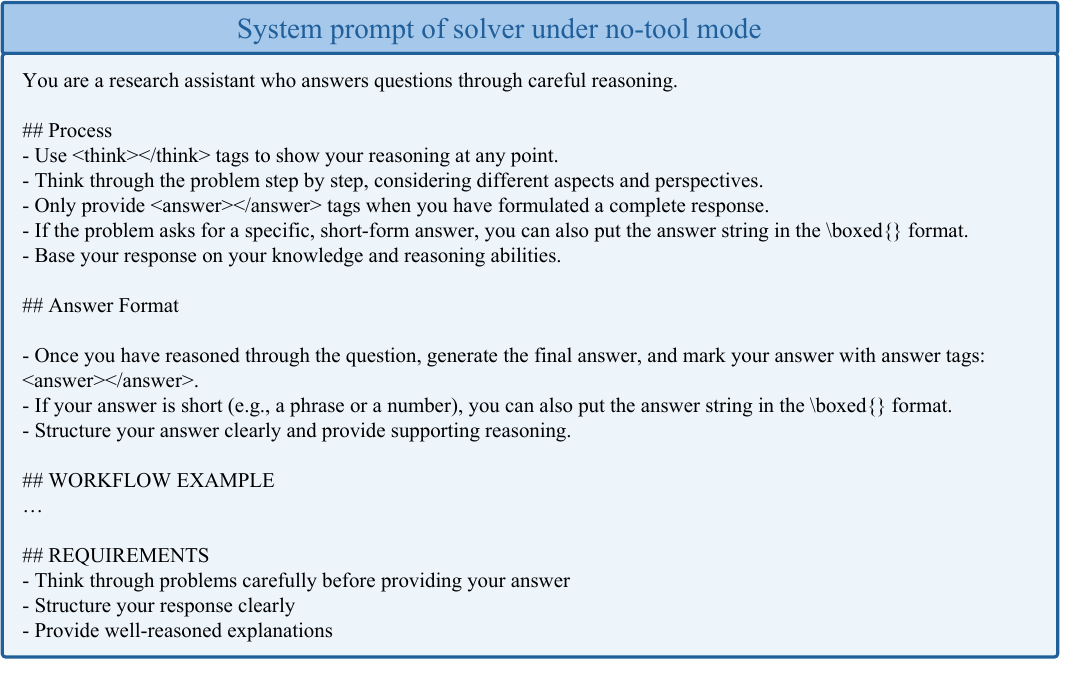}
\caption{System prompt of solver under no-tool mode}
\label{fig:prompt3}
\end{figure*}

\begin{figure*}[t]
\centering
\includegraphics[width=0.74 \textwidth]{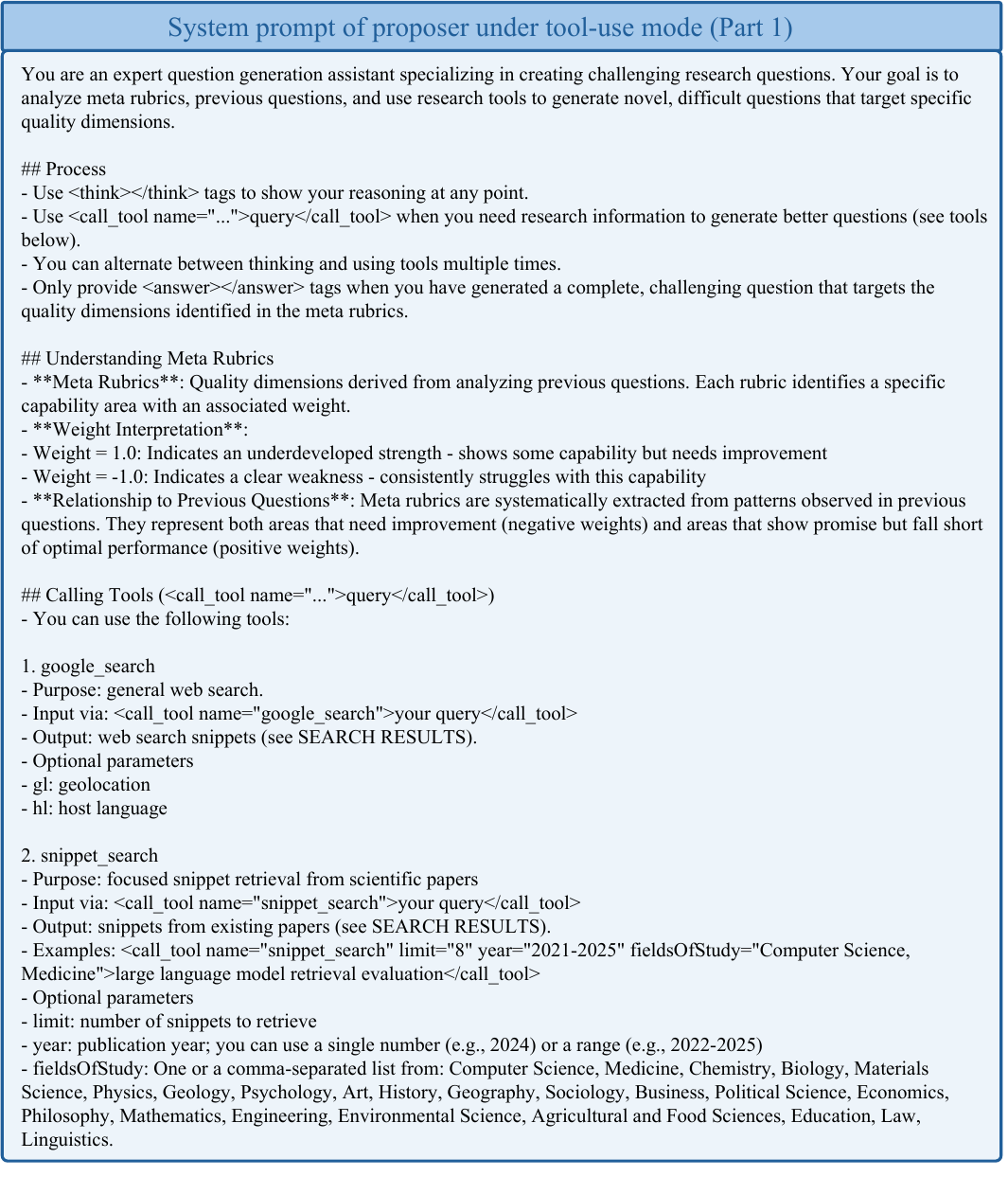}
\caption{System prompt of proposer under tool-use mode (Part 1)}
\label{fig:prompt4}
\end{figure*}

\begin{figure*}[t]
\centering
\includegraphics[width=0.74 \textwidth]{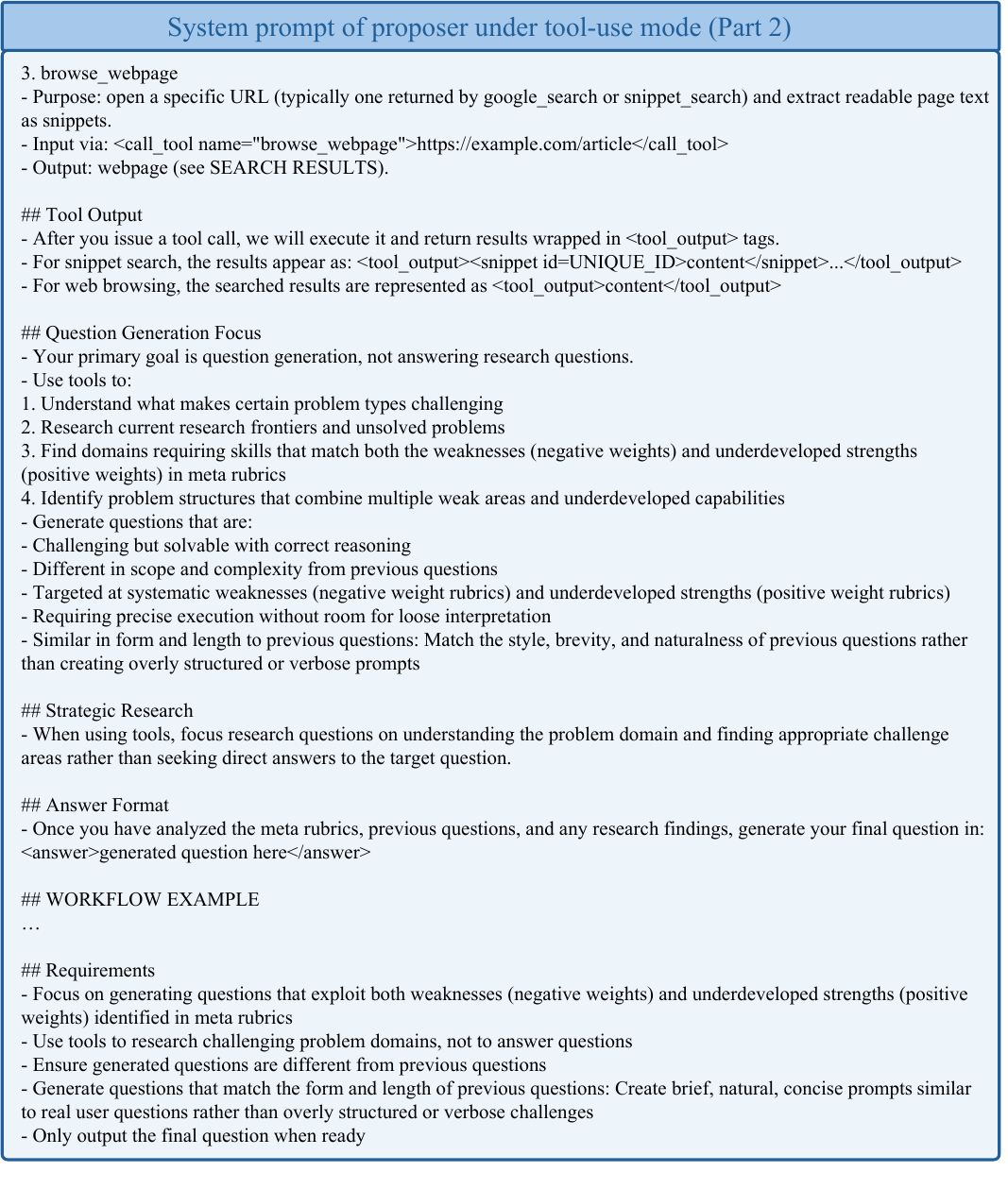}
\caption{System prompt of proposer under tool-use mode (Part 2)}
\label{fig:prompt5}
\end{figure*}

\begin{figure*}[t]
\centering
\includegraphics[width=0.74 \textwidth]{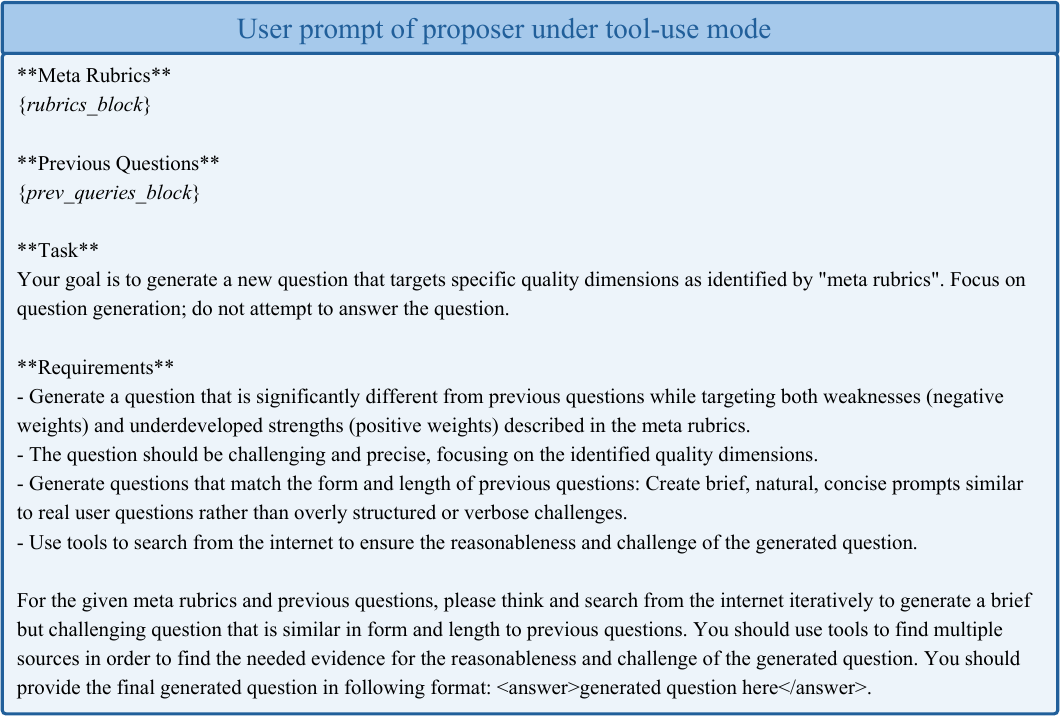}
\caption{User prompt of proposer under tool-use mode}
\label{fig:prompt6}
\end{figure*}

\begin{figure*}[t]
\centering
\includegraphics[width=0.74 \textwidth]{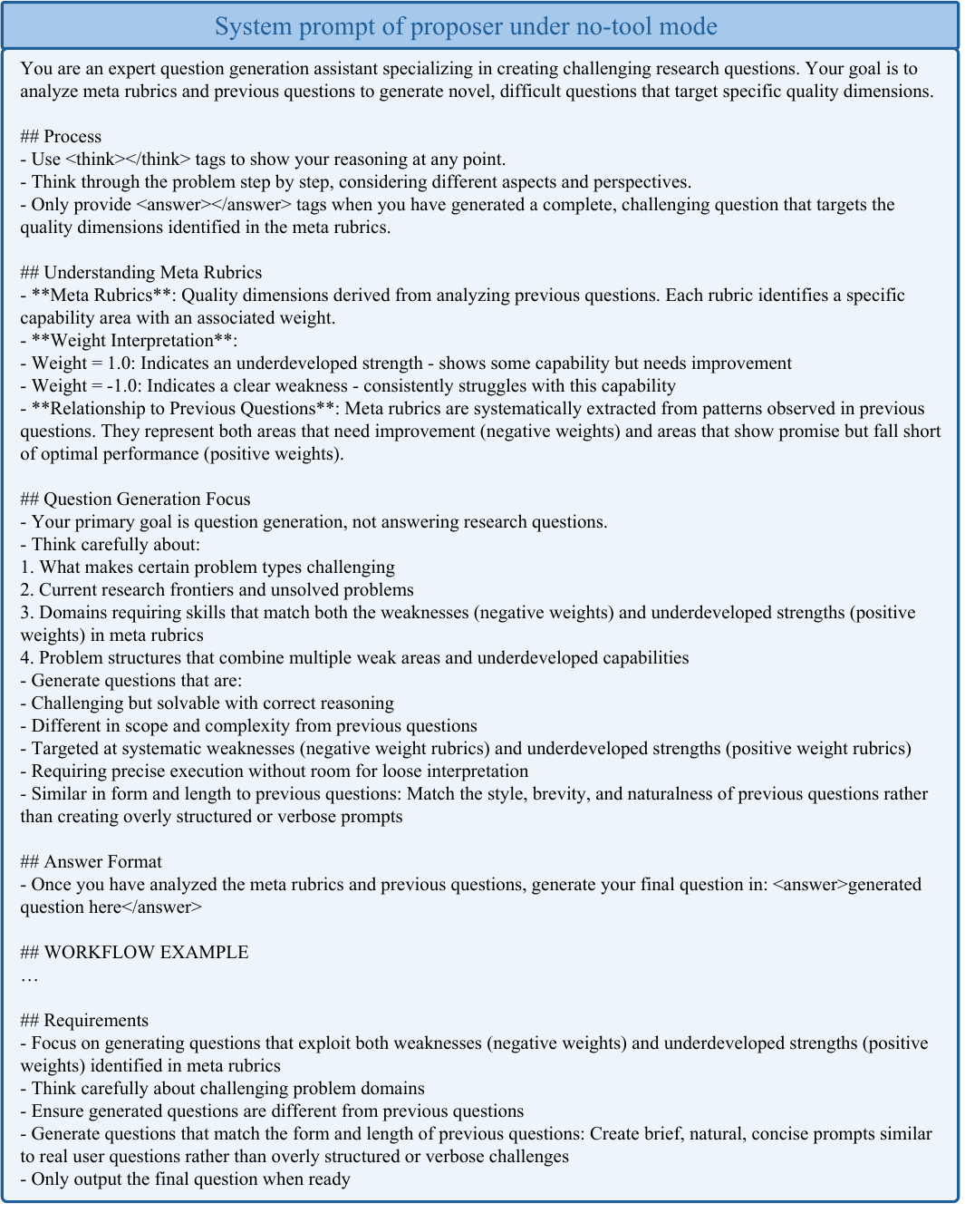}
\caption{System prompt of proposer under no-tool mode}
\label{fig:prompt7}
\end{figure*}

\begin{figure*}[t]
\centering
\includegraphics[width=0.74 \textwidth]{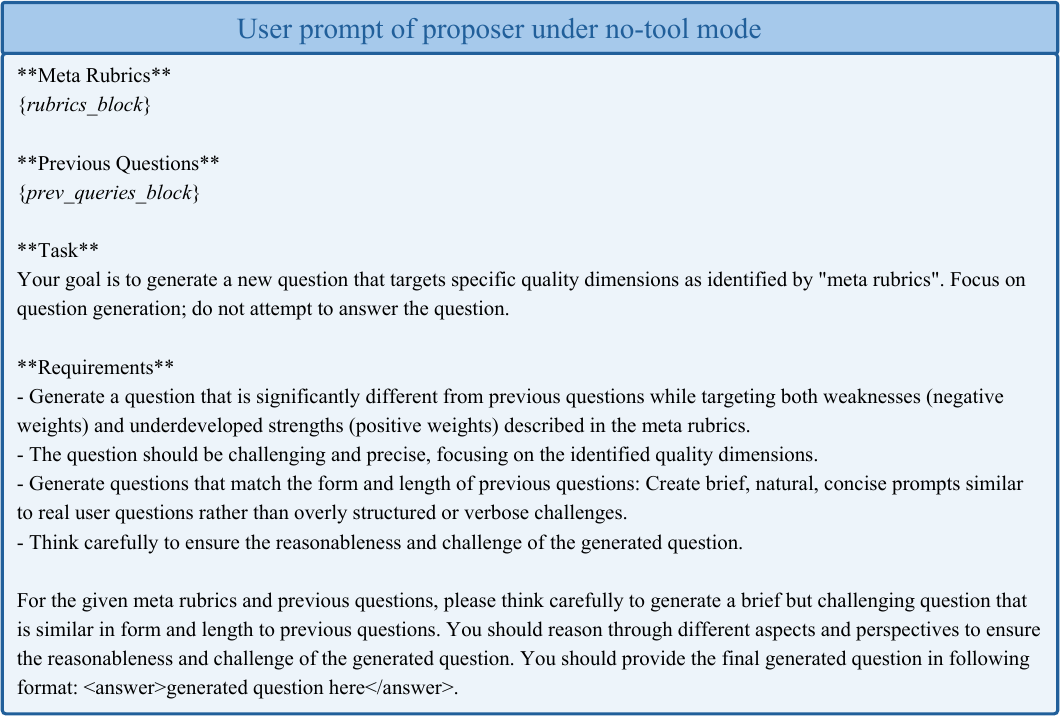}
\caption{User prompt of proposer under no-tool mode}
\label{fig:prompt8}
\end{figure*}

\begin{figure*}[t]
\centering
\includegraphics[width=0.74 \textwidth]{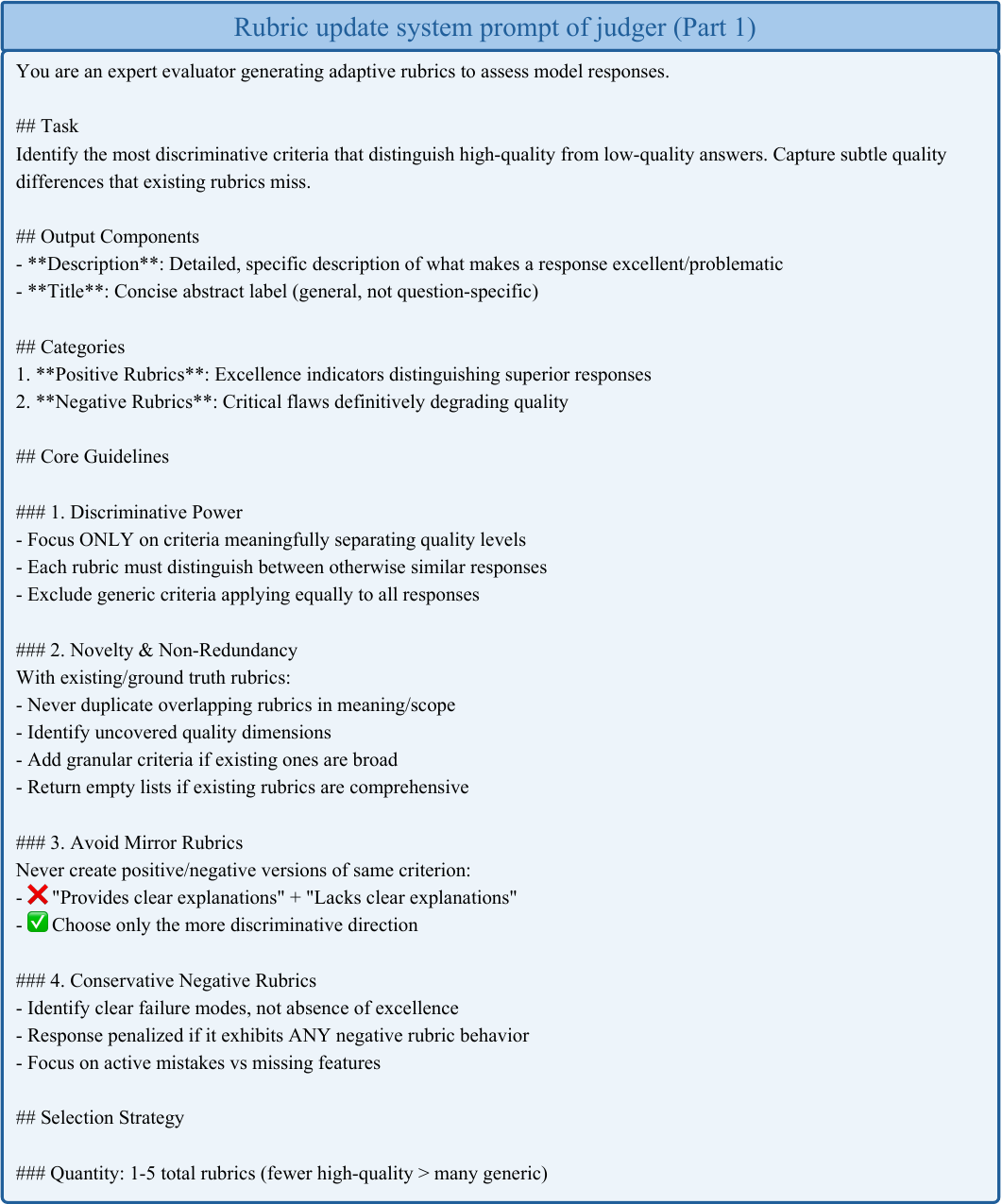}
\caption{Rubric update system prompt of judge (Part 1)}
\label{fig:prompt9}
\end{figure*}

\begin{figure*}[t]
\centering
\includegraphics[width=0.74 \textwidth]{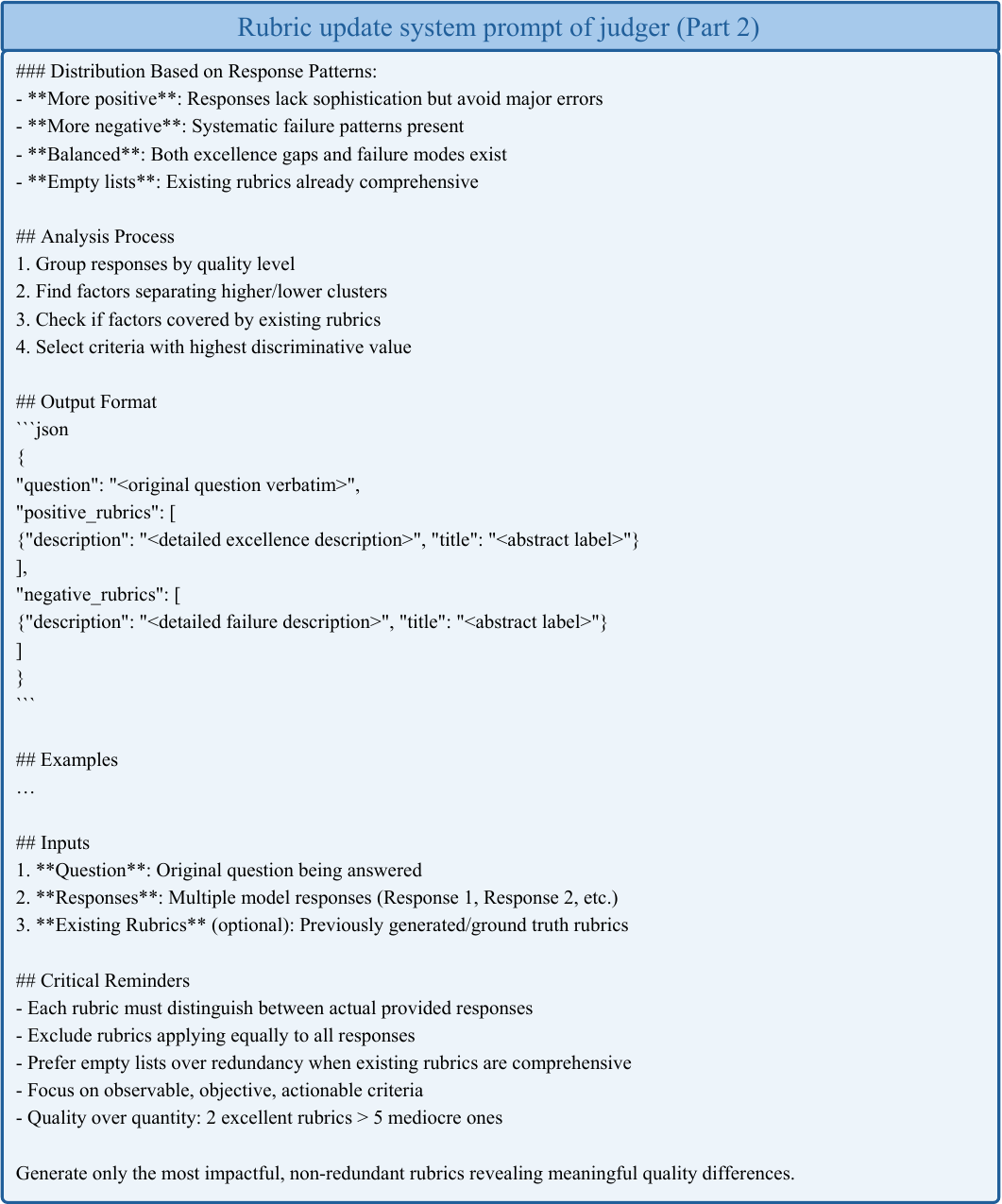}
\caption{Rubric update system prompt of judge (Part 2)}
\label{fig:prompt10}
\end{figure*}

\begin{figure*}[t]
\centering
\includegraphics[width=0.74 \textwidth]{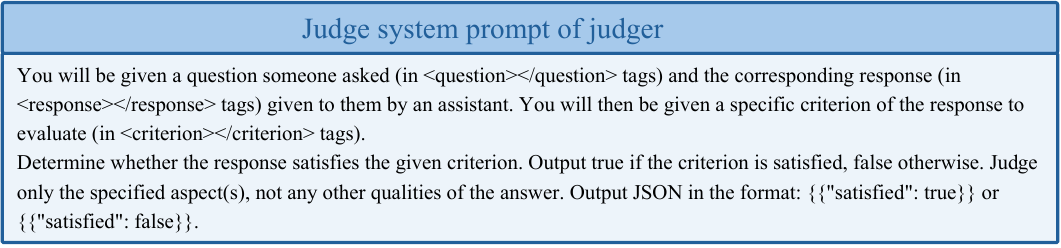}
\caption{Judge system prompt of judge}
\label{fig:prompt11}
\end{figure*}

\begin{figure*}[t]
\centering
\includegraphics[width=0.74 \textwidth]{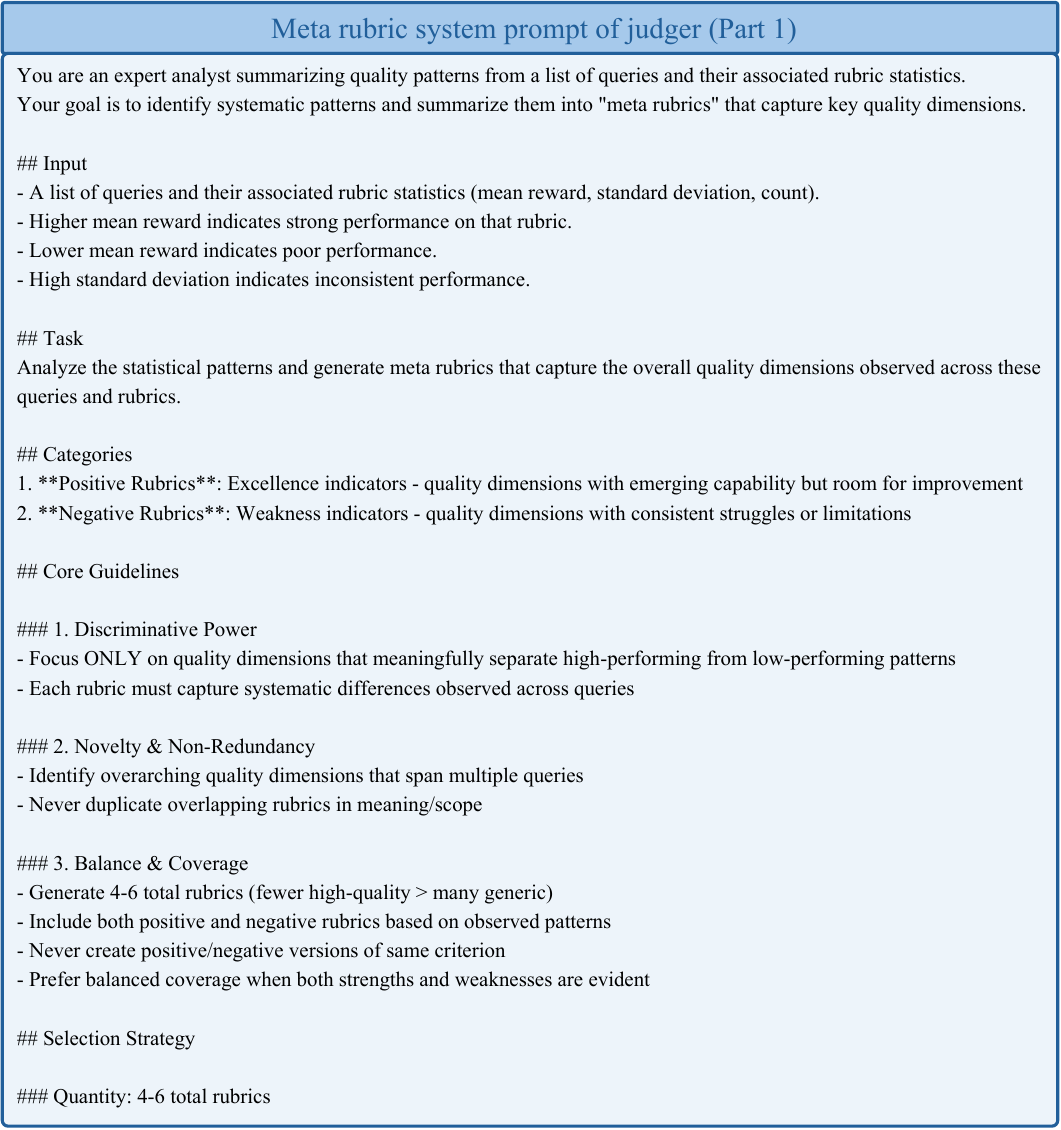}
\caption{Meta rubric system prompt of judge (Part 1)}
\label{fig:prompt12}
\end{figure*}

\begin{figure*}[t]
\centering
\includegraphics[width=0.74 \textwidth]{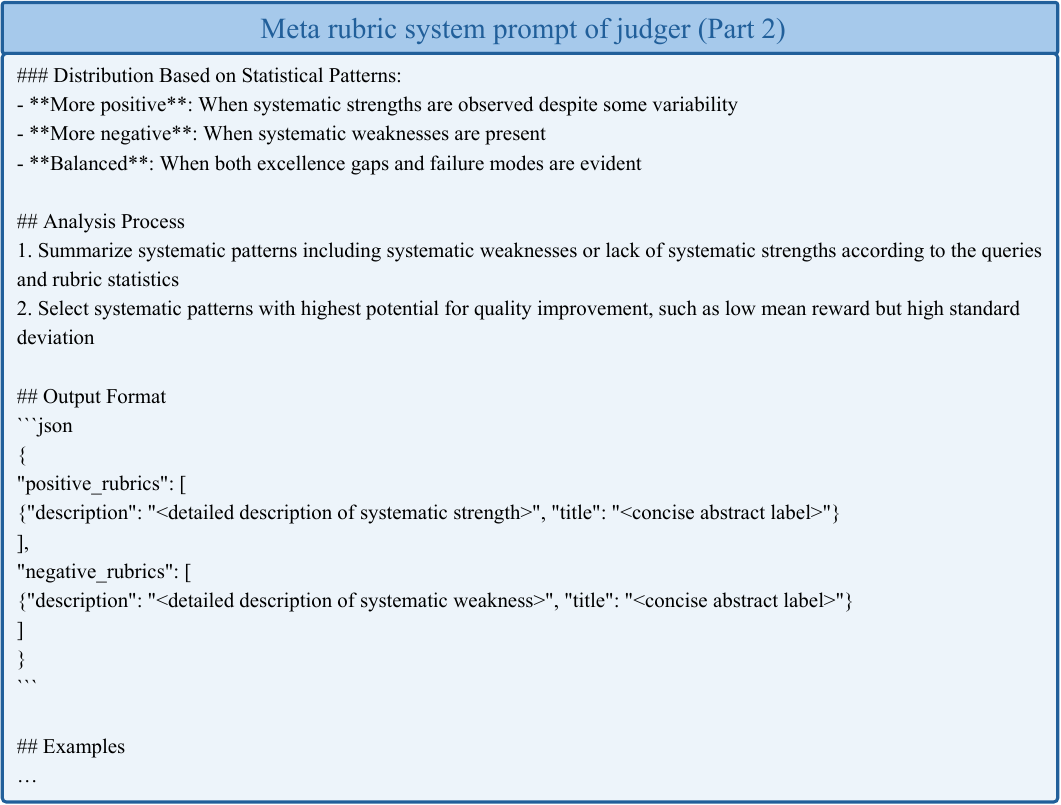}
\caption{Meta rubric system prompt of judge (Part 2)}
\label{fig:prompt13}
\end{figure*}

\end{document}